\documentclass[11pt]{article}

\usepackage{graphicx} % Required for inserting images

\usepackage{amssymb, amsmath, amsthm, amsfonts, bbm}
\usepackage{url, graphicx,caption,subcaption}
\usepackage[round]{natbib}

\usepackage{boxedminipage}
\usepackage{cleveref}

\usepackage{algorithm}
\usepackage{algpseudocode}

\usepackage{xspace}

\newtheorem{theorem}{Theorem}[section]

\newtheorem{lemma}[theorem]{Lemma}
\newtheorem{proposition}[theorem]{Proposition}

\newtheorem{assumption}[theorem]{Assumption}
\newcommand{\R}{\mathbb{R}}  
\newcommand{\E}{\mathbb{E}}
\newcommand{\cX}{\mathcal{X}}

\newcommand{\cR}{\mathcal{R}}
\newcommand{\cY}{\mathcal{Y}}
\newcommand{\cO}{\mathcal{O}}
\newcommand{\cF}{\mathcal{F}}
\newcommand{\cH}{\mathcal{H}}
\newcommand{\cP}{\mathcal{P}}
\newcommand{\Ber}{\mathrm{Ber}}
\newcommand{\cD}{\mathcal{D}}
\newcommand{\eps}{\epsilon}

\newcommand{\Ymin}{Y_{\mathrm{min}}}
\newcommand{\Ymax}{Y_{\mathrm{max}}}

\newcommand{\yt}{\tilde{y}}

\newcommand{\SM}{\mathbf{smCE}}

\newcommand{\norm}[1]{\left\|#1 \right\|_2}
\renewcommand{\eps}{\varepsilon}
\newcommand{\poly}{\mathrm{poly}}

\newcommand{\argmin}{\mathrm{argmin}}
\usepackage{xcolor}

\newcommand{\cFLip}{\cF_{\mathrm{Lip}}}
\def \endprf{\hfill {\vrule height6pt width6pt depth0pt}\medskip}
\renewenvironment{proof}{\noindent {\bf Proof} }{\endprf\par}

\newcommand{\algname}{Defensive Forecasting\xspace}
\newcommand{\ind}[1]{\mathbf{1}\left\{#1\right\}}

\title{In Defense of Defensive Forecasting}
\author{ Juan Carlos Perdomo\thanks{Departments of Computer Science \& Engineering and Data Science, New York University. Correspondence to \protect\url{j.perdomo.silva@nyu.edu}, \protect\url{brecht@berkeley.edu}}\; and Benjamin Recht\thanks{Department of Electrical Engineering and Computer Sciences, University of California, Berkeley.}}

\date{\today}

\begin{document}

\maketitle

\begin{abstract}
This tutorial provides a survey of algorithms for \algname, where predictions are derived not by prognostication but by correcting past mistakes. Pioneered by \citet{vovk2005defensive}, \algname frames the goal of prediction as a sequential game, and derives predictions to minimize metrics no matter what outcomes occur. We present an elementary introduction to this general theory and derive simple, near-optimal algorithms for online learning, calibration, prediction with expert advice, and online conformal prediction. 
\end{abstract}


\section{Introduction}\label{sec:intro}

From sports to politics, from the stock market to prediction markets, from cancer detection to sequence completion, prediction is a big business. But how can someone get in on the action? A forecaster is only as good as their record, so they must demonstrate that their predictions are prescient. To make good predictions, it seems like you need some level of clairvoyance to see what the future holds. But what if forecasters can cleverly cook their books to make their predictions look good? What if they could make predictions that correct the errors you made in the past? In this case, they wouldn’t need to know anything about the future. They'd just need to know how to do proper accounting. In this survey, we describe a simple, general strategy for such strategic accounting, Defensive Forecasting. 

Defensive Forecasting was first proposed by Vovk, Takemura, and Shafer as a game-theoretic strategy for prediction. \citet{vovk2005defensive} assume that nature, which is producing future outcomes, is adversarial. A forecaster’s job is then to make a prediction so that no matter what the adversarial nature does, the forecaster comes out ahead. The key is to utilize the sequential interaction with nature, so that forecasters sequentially update their predictions as nature reveals outcomes. This notion of predictions and decisions as game theoretic goes back to~\citet{wald1945statistical}. It was revisited in the 1990s in a line of influential work on calibration initiated by \citet{foster1998asymptotic} and extended in \citep{sandroni2003calibration, lehrer2001any, fudenberg1999easier}. In modern learning theory, it has been a principle for algorithm design, motivating applications of game-theoretic tools like Blackwell Approachability~\citep{abernethy2011blackwell,perchet2013approachability}, Approximate Dynamic Programming~\citep{rakhlin2012relax,RakhlinLectureNotes}, or Fixed Point Theory \citep{foster2021forecast}. However, Defensive Forecasting is considerably simpler than all of these approaches. It uses only a restricted set of moves specifically designed to correct past errors. The game of robustly optimizing against an adversarial nature collapses into simple bookkeeping. 

We write this tutorial with two goals in mind. The first is to provide an accessible introduction \algname, a collection of powerful and underappreciated techniques for sequential prediction. Rather than thinking of predictions as having anything to do with the future, these algorithms view forecasting as a means to correct past mistakes. We work through examples that demonstrate how \algname yields simple and powerful algorithms for a variety of interesting problems, including online learning, debiasing, calibration, prediction with expert advice, and conformal prediction. 

Having presented these technical ideas, the second goal of our tutorial is to answer the conceptual question we laid out at the beginning: what \emph{is} a prediction in the first place? What makes a prediction ``good?'' We aim to demonstrate that if you can clearly specify your objectives and describe your epistemic commitments regarding what is predictable, you can derive a defensive forecasting strategy that provably optimizes the proposed evaluation.

We proceed by first examining the simple case of predicting the probability of bits in a sequential fashion. This motivates a general algorithm for \algname that we discuss in \Cref{sec:meta_algorithm}. We present a special case of \algname in \Cref{sec:dmm} that yields algorithms for online decision making (\Cref{sec:risk_minimization}) and prediction with linear combinations of features (\Cref{sec:linear_bandits}). These sections illustrate that \algname has a close relationship to the notion of Outcome Indistinguishability \citep{dwork2021outcome}. In essence, the probabilistic predictions are chosen so that that the analyst could proceed assuming that the outcomes had been sampled with those probabilities.

We next follow~\citet{vovk2007k29} and generalize \algname to kernel spaces (\Cref{sec:kdmm}). This will yield simple methods for calibration. Indeed, we show how many different notions of calibration can be achieved by \algname in  \Cref{sec:calibration}. We also show how \algname algorithms can yield optimal results for prediction with expert advice (\Cref{sec:experts}) and for the problem of computing quantiles in an online fashion (\Cref{sec:quantiles}).

We strive to keep this technical overview as elementary as possible, seeking the simplest and most direct algorithms with the shortest analyses. As a result, we don't strive to make every analysis as tight as possible though we point out a few cases where the algorithms are optimal. We focus on providing intuitions and highlighting the main ideas. Most of the mathematics needed consists of algebraic manipulations and rudimentary probability. We have a brief digression into kernel methods, but anyone familiar with kernel methods in machine learning will find this section approachable. Throughout, we provide pointers to the relevant literature for readers interested in the latest, most sophisticated results. 

\section{Rudiments of \algname for Predicting Events}\label{sec:bit_prediction}

Suppose we want to predict the likelihood that a certain event will occur based on observations of past events. For example, if we are going to predict whether a basketball player will make their next free throw, we will use their past success rate as a guess for the probability of the next shot. We can abstract this as observing a sequence of $T$ bits (assigning a $1$ if they make the shot and a $0$ otherwise), and wanting to predict the likelihood that the next bit will be a $1$ (i.e., will they make the shot?). A reasonable heuristic guess for that likelihood is the average of the first $T$ bits. If there were far more ones than zeros, it is sensible to assume the next bit will also be more likely a one than not.

Note that in this elementary prediction exercise, there are two components. First, the inductive assumption that rates in the past are indicative of likelihoods in the future. Second, the observation that an elementary algorithm can compute the past rate. The inductive assumption and the algorithm are effectively independent! You could calculate the rate of the past bits, no matter how the next bit relates to the previous bits. Moreover, the extent to which the average of the first $T$ bits is a reliable signal of the next bit cannot be determined based off any statistics of the bits we have seen so far. The only separation between viewing this summary as meaningless historical bookkeeping versus an insightful claim about the future is a fundamental, inductive assumption that the processes generating our data remain stable over time.  Furthermore, this defining, inductive assumption that the past looks like the future is fundamentally unrelated to any formal guarantees we prove about the performance of online algorithms on the realized sequence. 

Let's now formalize the sequential bit prediction problem and derive algorithms with such formal guarantees. We'd like to predict a sequence of bits, $y_1, y_2, \ldots, y_T$. We are allowed to use the previous $t$ samples to predict $y_{t+1}$. What should we predict? Let $p_t$ denote the prediction of the $t$th bit. As we've already mentioned, what we predict is determined by how we will be scored. Thus, we need to first describe an evaluation and then see how to make predictions to ace the prescribed test.

As a simple example, let's say that we will evaluate the predictions according to the absolute error metric:
$$
    \left| \frac{1}{T} \sum_{t=1}^T p_t - \frac{1}{T} \sum_{t=1}^T y_t \right|\,. 
$$
In this metric, we can let $p_t$ be real valued and think of $p_t$ as the probability that $y_t=1$. In this framing, the expected number of times $y_t=1$ is exactly $\sum_{t=1}^T p_t$. The realized number of times $y_t=1$ is of course $\sum_{t=1}^T y_t$. If the online algorithm has low absolute error, we can say that its predictions forecast the true number of positive events on average.

To motivate the general strategy of \algname, suppose we could show that our prediction algorithm satisfies the inequality:
\begin{equation}\label{eq:diag-ub}
    \left(  \sum_{t=1}^T p_t - \sum_{t=1}^T y_t \right)^2 \leq  \sum_{t=1}^T (y_t-p_t)^2 \,. 
\end{equation}
Then, since $|y_t-p_t|\leq 1$, the inequality above inequality implies
$$
    \left| \frac{1}{T} \sum_{t=1}^T p_t - \frac{1}{T} \sum_{t=1}^T y_t \right| \leq \frac{1}{\sqrt{T}}\,. 
$$
In this case, for large $T$, the prediction algorithm would have low error. 

We can achieve \Cref{eq:diag-ub} inductively. Suppose the bound was satisfied for $t\leq T-1$. Define
$$
S_t = \sum_{t=1}^T (y_t-p_t).
$$
Then, for the final step $T$, we have,
$$
    S_T^2 = (S_{T-1} + y_T-p_T)^2 = S_{T-1}^2  + 2(y_T-p_T) S_{T-1} + (y_T-p_T)^2\,.
$$
If we choose $p_T$ so that the cross term $2(y_T-p_T) S_{T-1}$ is always non-positive, we will have achieved~\eqref{eq:diag-ub} by induction. But making this cross term nonpositive is straightfoward: if $S_{T-1}$ is negative, setting $p_T=0$ yields a nonpositive cross term no matter whether $y_T$ is 1 or 0. Similarly, if $S_{T-1}$ is nonnegative, $p_T=1$ yields a non-positive cross term no matter what $y_T$ ends up being. 

In sum, we get a simple algorithm for choosing the next prediction. 
We can initialize by predicting $p_1=0$. Then for each subsequent $t$, we can predict $p_t=1$ if $S_{t-1}\geq 0$ and $p_t=0$ otherwise.

Looking at what this algorithm actually does is instructive: Since $p_1=0$, at step 2 of the algorithm, $S_1=y_1$ and hence $p_2=y_1$. Similarly, at step 3, $S_2= (y_2-p_2)+(y_1-p1) = y_2$, and hence $p_3=y_2$. At each time step, we just predict the bit we saw in the previous time step. The algorithm guesses that the future will be the same as the present. Though this aggressive strategy seems to rest too heavily on the immediate past, is it not different in spirit from using a running average of the past few time steps to predict the next bit. However, note that the algorithm was not derived through some metaphysical equating of the past and the future. Instead, the evaluation metric implied a straightforward algorithmic solution of correcting the error accumulated by the last observation. Rather than making any assumptions about the future, low error can be achieved by correcting mistakes of the past.

Now, experts might note that the error of $1/\sqrt{T}$ is suboptimal for learning means. Perhaps this algorithm could be improved by being less aggressive. A more careful analysis shows this is not the case. We have
$$
    \left| \frac{1}{T} \sum_{t=1}^T p_t - \frac{1}{T} \sum_{t=1}^T y_t \right| =  \left| \frac{1}{T} (0 + \sum_{t=1}^{T-1} y_t) - \frac{1}{T} \sum_{t=1}^T y_t \right| =  \frac{y_T}{T} \leq \frac{1}{T}\,.
$$
This $\frac{1}{T}$ error rate is considerably better. Given that $y_T$ is revealed after $p_T$, an error of $\frac{1}{2T}$ is unavoidable for any algorithm. Hence, up at most a small constant, \algname is optimal for this error metric.

\section{\algname: A Meta Algorithm}
\label{sec:meta_algorithm}

Let's zoom out and discuss a broad, meta-strategy for \algname, generalizing the bit prediction example from the last section to encompass a family of powerful algorithms for various prediction problems.

At each time $t$ we observe a context vector $x_t \in \cX$, make a prediction $p_t \in \cP$, and then see the realized outcome $y_t\in \cY$. Rather than wanting $p_t$ to match $y_t$, we aim to make predictions $p_t$ that minimize,
$$
	\left\|\frac{1}{T}\sum_{t=1}^T F(x_t,p_t,y_t)\right\|_{2},
$$
where $F$ is some specified vector-valued function. Note that as long as $p_t$ is chosen so that
\begin{equation}\label{eq:foundational}
\sup_{y\in\mathcal{Y}} \left\langle F(x_t,p_t,y),  \sum_{s=1}^{t-1} F(x_s,p_s,y_s)\right\rangle \leq 0
\end{equation}
we have
\begin{equation}\label{eq:diagonal-bound}
\norm{\sum_{t=1}^T F(x_t,p_t,y_t)}^2 \leq \sum_{t=1}^{T} \norm{F(x_t,p_t,y_t)}^2 \,,
\end{equation}
and hence, if $\norm{F(x,p,y)}\leq M$ for all triples $(x,p,y) \in \cX \times \cP \times \cY$,
$$
\norm{\frac{1}{T}\sum_{t=1}^T F(x_t,p_t,y_t)} \leq \sqrt{\frac{\sum_{t=1}^T \norm{F(x_t,p_t,y_t)}^2}{T^2}} \leq \frac{M}{\sqrt{T}}\,.
$$
To see why \eqref{eq:diagonal-bound} holds, we can apply induction:
\begin{align*}
\norm{\sum_{t=1}^T F(x_t,p_t,y_t)}^2 &= \norm{\sum_{t=1}^{T-1} F(x_t,p_t,y_t)}^2
+ 2\left\langle F(x_t,p_t,y_t),  \sum_{s=1}^{T-1} F(x_s,p_s,y_s)\right\rangle + \norm{F(x_t,p_t,y_t)}^2\\
&\leq \sum_{t=1}^{T-1} \norm{F(x_t,p_t,y_t)}^2+ \norm{F(x_t,p_t,y_t)}^2\,.
\end{align*}
The final inequality follows by the inductive hypothesis and what we will refer to as the fundamental condition of \algname, \Cref{eq:foundational}. 
% Furthermore, note that this entire derivation also holds if instead of predicting a single $p_t$, we picked a distribution $\cD_t$ over $p \in \cP$
% \begin{equation}\label{eq:foundational_dist}
% \sup_{y\in\mathcal{Y}} \left\langle \sum_{p \in \cP} \cD_t(p)F(x_t,p_t,y),  \sum_{s=1}^{t-1} F(x_s,p_s,y_s)\right\rangle \leq 0. 
% \end{equation}
% In this case, we would instead arrive at the following, in expectation, guarantee:
% $$
% \norm{\frac{1}{T}\sum_{t=1}^T \E_{p_t \sim \cD_t}F(x_t,p_t,y_t)} \leq \frac{M}{\sqrt{T}}\,.
% $$
We summarize this meta-algorithm in Algorithm~\ref{alg:def-book}. 

\begin{algorithm}[t!]
\caption{\algname}\label{alg:def-book}
\begin{algorithmic}[1]
\For{$i=1,\ldots, T$}
\State Receive context $x_t$.
\If{$t=1$}
\State Predict $p_1$ using initialization rule.
\Else
\State Predict $p_t$ such that $\sup_{y\in\mathcal{Y}} \left\langle F(x_t,p_t,y),  \sum_{s=1}^{t-1} F(x_s,p_s,y_s)\right\rangle \leq 0$. \label{book-game} 
\EndIf
\State Receive outcome $y_t$
\EndFor
\end{algorithmic}
\end{algorithm}

The key question is thus when does $p_t$ satisfying the fundamental \algname condition (aka Line~\ref{book-game} in Algorithm~\ref{alg:def-book}) exist? In the simplest form, we need to know that for every $x \in \cX$ and $z \in \R^d$, there exists a $p\in \cP$ such that for all $y \in \cY$,
\begin{equation}\label{eq:variational}
\left\langle F(x_t,p_t,y), z \right\rangle\leq 0.
\end{equation}
There are a variety of ways to solve such nonlinear feasibility problems. When $F$ is linear in $p$, these problems can be solved by Blackwell Approachability \citep{Blackwell56}. For example, \citet{foster1999proof} used Blackwell Approachability to solve a particular flavor of calibration problem related to those we discuss in \Cref{sec:calibration}. \citet{abernethy2011blackwell} has shown that online linear regret maximization is equivalent to Blackwell Approachability. \citet{RakhlinLectureNotes} show how to apply Blackwell Approachability to solve more challenging bit prediction problems. Chapter 7 of \citet{cesaBianchi2006prediction} and the survey by \citet{perchet2013approachability} also highlight several other applications of Blackwell Approachability to online learning.

Other tools from variational analysis are also likely applicable to solving problems of the form \Cref{eq:variational}. For example, in its most general form, this expression is a variational inequality~\citep{hartman1966some}, and techniques from this branch of mathematical optimization may be applicable. Recent work on calibration by \citet{foster2021forecast} uses an analysis in terms of outgoing fixed points \citep{border1985fixed} to solve a similar feasibility problem.

While all of these powerful mathematical tools that make such inequalities solvable, in this tutorial, we focus on cases of functions $F$ where we can always satisfy basic inequalities like~\Cref{eq:foundational}. These cases are simple enough to yield elementary proofs without any appeals to convex analysis or topology.  

In fact, all of the algorithms we derive here have the same form. We first find an efficiently computable function $S_t: [0,1]\rightarrow \R $ that summarizes the mistakes of the past. If $S_t(1)$ is nonnegative, we predict $1$. If not, we check $S_t(0)$. If it is non-positive, we predict $0$ If neither condition holds, then we are guaranteed that $S_t(p)$ has a root in $[0,1]$. We find this root by binary search, and this root then serves as our prediction. We call this subroutine \emph{anticorrelation search}, and summarize it in \Cref{alg:anti-search}.

Somewhat surprisingly, algorithms based on this form of anticorrelation search also suffices to let us recover near-optimal results from online learning, conformal prediction, and calibration with elementary arguments. 

%More sophisticated mathematics could be employed to solve \algname problems, but our goal in this manuscript is to show how many existing applications can already be derived by elementary algebra and bookkeeping.

\begin{algorithm}[t!]
\caption{Anticorrelation Search}\label{alg:anti-search}
\begin{algorithmic}[1]
\State Given summary function $S:[0,1]\rightarrow \R$.
\If{$S(1) \geq  0$}
\State Return $p = 1$.
\ElsIf{$S(0) \leq 0$}
\State Return $p=0$.
\Else
\State Run binary search on $S(\cdot)$ to find $p$ satisfying $S(p)=0$.
\State Return $p$.
\EndIf
\end{algorithmic}
\end{algorithm}

\section{Defensive Moment Matching}\label{sec:dmm}
Let's first consider when the function $F$ takes the form
\begin{equation}\label{eq:moment-vector}
    F(x,y,p) = (y-p) \Phi(x,p)
\end{equation}
where $\Phi$ is a vector-valued function that is continuous in $p$ for $p \in \cP = [0,1]$. The entries in $\Phi(x,p)$ represent different features of the pair $(x,p)$.

We illustrate how predictions yielding low norm $F$ satisfy an intriguing notion of predictive validity: The predictions act as if they were ``true probabilities'' that generated the outcomes $y_t$.

We record the following lemma, which will be valuable throughout.

\begin{lemma}\label{lemma:oi}
Let $F(x,p,y) = (y-p)\Phi(x,p)$ and suppose that for some constant $C$,
\begin{align}
\label{eq:regression_F}
   \norm{ \frac{1}{T} \sum_{i=t}^T (y_t-p_t) \Phi(x_t,p_t)} \leq \frac{C}{\sqrt{T}}\,.
\end{align}
Then, for any function $f(x,p,y)$ such that, $f(x,p,1) - f(x,p,0) = \langle v, \Phi(x,p) \rangle$ where $v$ is an arbitrary fixed vector, we have
\begin{align}
\label{eq:online_oi}
   \left|\frac{1}{T} \sum_{i=t}^T f(x_t,p_t,y_t) -  \frac{1}{T} \sum_{i=t}^T \E_{\yt_t\sim \Ber(p_t)}[f(x_t,p_t,\yt_t)]  \right| \leq \frac{C \norm{v}}{\sqrt{T}}\,.
\end{align}
\end{lemma}
Lemma~\ref{lemma:oi} asserts that if Defensive Forecasting makes \Cref{eq:regression_F} small, then we can effectively treat a large set of functions of $f(x,p,y)$ evaluated at the true outcomes as if the $y_t$ were sampled from a Bernoulli distribution with mean $p_t$. That is, for all intents and purposes, we can retrospectively pretend the $y_t$ are biased coin flips with the biases $p_t$ we wrote down as our predictions. Following \cite{dwork2021outcome} we will refer to predictions satisfying \Cref{eq:online_oi} as being (online)  outcome indistinguishable with respect to the set of functions $f$. Later on in sections \Cref{sec:risk_minimization} and \Cref{sec:calibration} we will describe various interesting classes of functions $f$ that can be written this way, $f(x,p,1) - f(x,p,0) = \langle v, \Phi(x,p) \rangle$.

% For example, if we let $\Phi(x,p)=x$ for $x\in \R^d$, setting $v$ to be the $j$th standard basis vector, we get that the observed correlation between any feature $x^{(j)}$ and $y$, $f(x,p,y) = x^{(j)}y$ is equal to its expected value $\E_{\yt\sim p}[x^{(j)}\yt] = x^{(i)}p$.

As was also emphasized by Vovk and Shafer, the probabilities here are for bookkeeping predictions. \algname does not care whether or not the $y_t$ are random. In our online setup, they can be chosen completely arbitrarily, even with knowledge of the forecast $p_t$.  There may not be any patterns relating future and past outcomes whatsoever. Yet, \algname looks at the past outcomes to construct a prediction where we can pretend that the next bit was sampled from our prediction, no matter what the actual revealed outcome is. 

%While this is quite counterintuitive
%
%This claim is quite counterintuitive and seems flat-out false. How can we generate forecasts that look like they truly generated the data, if there is no ``true'' underlying data-generating process in the first place? Defensive forecasting provides an almost trivial algorithm that does exactly that.

Despite its counterintuitive relationship to probability, Lemma~\ref{lemma:oi} has a simple proof. For any triplet $(x,p, y)$ where $y \in \{0,1\}$, we can write $f$ as a linear function of $y$,
\begin{align*}
    f(x,p,y) = y f(x,p, 1) + (1-y) f(x,p, 0)  = y[f(x,p, 1) - f(x,p, 0)] + f(x,p,0). 
\end{align*}
A similar rewriting holds in expectation when $y$ is sampled from a Bernoulli distribution,
    \begin{align*}
        \E_{y \sim \Ber(p)} f(x,p, y) &= p f(x,p, 1) + (1-p) f(x,p,0)  = p[f(x,p, 1) - f(x,p, 0)] + f(x,p, 0).
    \end{align*}
Taking their difference, the $f(x,p,0)$ term cancels out and we get that,
\begin{align*}
    \frac{1}{T}\sum_{i=t}^T f(x_t,p_t,y_t) -   \frac{1}{T}\sum_{i=t}^T \E_{\yt_t\sim \Ber(p_t)}[f(x_t,p_t,\yt_t)] &= \frac{1}{T}\sum_{t=1}^T (f(x_t,p_t,1) -f(x_t,p_t,1))(y_t-p_t) \\ 
    & = \frac{1}{T}\sum_{t=1}^T \langle v, \Phi(x_t,p_t) \rangle  (y_t-p_t),
\end{align*}
where we used the assumption $f(x,p,1) - f(x,p,0) = \langle v, \Phi(x,p) \rangle$. We can bound this last term in terms of the norm of $v$ and a term made small by \algname.
\begin{align*}
\frac{1}{T}\sum_{t=1}^T \langle v, \Phi(x_t,p_t) \rangle  (y_t-p_t) =  \langle v, \frac{1}{T}\sum_{t=1}^T \Phi(x_t,p_t)(y_t-p_t) \rangle  \leq \norm{v} \norm{\frac{1}{T}\sum_{t=1}^T \Phi(x_t,p_t)(y_t-p_t) }.
\end{align*}
This proves the Lemma.

Let's now derive a variant of \algname that guarantees the bound in  \Cref{eq:regression_F}. Define,
$$
    S_t(p) = \sum_{s=1}^{t-1} \langle \Phi(x_t,p), \Phi(x_s,p_s) \rangle (y_s-p_s)\,.
$$
From \Cref{eq:foundational}, \algname chooses $p_t$ such that, 
\begin{align}
\label{eq:fundamental_ineq_regr}
\sup_{y\in\mathcal{Y}} \left\langle F(x_t,p_t,y),  \sum_{s=1}^{t-1} F(x_s,p_s,y_s)\right\rangle = \sup_{y\in \{0,1\}} (y-p_t) \cdot S_t(p_t) \leq 0.
\end{align}
This guarantee is only modestly harder to achieve than it was for bit prediction. If $S_t(1)\geq 0$, then we must have that $(y-1)S_t \leq 0$ for all $y$. Therefore, choosing $p_t=1$ would satisfy this condition. If this isn't the case, we can check if $S_t(0)\leq 0$. In this case, we'd have $(y-0)S_t(0) \leq 0$ for all $y$, and choosing $p_t=0$ would suffice. If neither of these conditions holds, continuity of $S_t$ implies there exists a $p \in (0,1)$ with $S_t(p)=0$. This $p$ would then imply the above inequality.  In sum, running anticorrelation search (\Cref{alg:anti-search}) on the function $S_t(p)$, yields a prediction satisfying \Cref{eq:fundamental_ineq_regr}.

This procedure, originally introduced by \citet{vovk2005defensive}, is summarized in Algorithm~\ref{alg:dmm}. Since the prediction $p_t$ satisfies \Cref{eq:fundamental_ineq_regr} and since $|y_t-p_t| \leq 1$, the analysis from \Cref{sec:meta_algorithm} shows 
\begin{align*}
    \norm{\frac{1}{T} \sum_{t=1}^T \Phi(x_t,p_t)(y_t-p_t)} \leq \frac{M}{\sqrt{T}}.
\end{align*}
where $M=\sup_{x,p} \norm{\Phi(x,p)}$.

Before running through several applications of this form of \algname, we quickly highlight that we can implement it efficiently. Though we presented \Cref{alg:anti-search} with exact root finding,  one can use approximate root finding and still yield an $O(1/\sqrt{T})$ guarantee. Indeed, as long as $|S_t(p_t)| \leq \eps$ for $\eps \leq 1 / \poly(t)$, a $O(1/\sqrt{T})$ bound will hold. We refer the reader to \cite{kernelOI} for details. We can find $\epsilon$-approximate roots by binary search with at most $\lceil \log(1/\eps) \rceil$ many evaluations of $S_t(p)$. Furthermore, if computing $\Phi(x,p)$ takes time $\cO(d)$ for $\Phi(x,p)\in \R^d$, then by maintaining the running sum, $\sum_{s=1}^{t-1} \Phi(x_s,p_s)(y_s-p_s)$, we can compute $S_t$ in time $\cO(d)$ at each time step $t$. Therefore, each $p_t$ can be computed in time $\widetilde{\cO}(d)$.

\begin{algorithm}[t]
\caption{\algname for Matching Empirical Moments}\label{alg:dmm}
\begin{algorithmic}[1]
\State Define $S_t(p) = \sum_{s=1}^{t-1} \langle \Phi(x_t,p), \Phi(x_s,p_s) \rangle(y_s -p_s)$.
\State Run anticorrelation search (\Cref{alg:anti-search}) on $S_t$ to find $p_t$.
\end{algorithmic}
\end{algorithm}

\section{Risk Minimization}\label{sec:risk_minimization}

As a first application, let's describe a simple problem that seems like it should be impervious to a \algname strategy. Suppose we are utility maximizers and want to minimize loss by choosing actions over time. Our goal is to accrue low regret,
\begin{align}
\label{eq:constant_action_regret}
    	\sum_{t=1}^T \ell(a_t, y_t) - \min_{a_\star}\sum_{t=1}^T \ell(a_\star, y_t) =  o(T),
\end{align}
where $a_t$ is the action chosen at time $t$ and $a_\star$ is the best constant action possible having known the sequence of $y_t$ in advance. This regret guarantee implies that the difference in average loss incurred between our actions $a_t$ and the best fixed action in hindsight goes to zero over time, 
$$
	\lim_{T \rightarrow \infty}\left|\frac{1}{T}\sum_{t=1}^T \ell(a_t, y_t) - \frac{1}{T} \sum_{t=1}^T \ell(a_\star, y_t) \right|= 0
$$
How could we achieve this? Consider the following thought experiment. If $y_t$ was truly random and sampled from a Bernoulli distribution with parameter $p_t$, then the optimal action $\pi(p_t)$ would be the one that minimizes the conditional expectation,
\begin{align}
\label{eq:pi_def}
	\pi(p_t) := \arg \min_a \E_{y_t\sim \operatorname{Ber}(p_t)} [\ell(a,y_t)].
\end{align}
The function $\pi(\cdot)$ is often very simple. For instance if $\ell(a,t)$ is the squared loss $(y-a)^2$, then $\pi(p_t) = p_t$. And if $\ell(a,t)$ is the 01 loss $1\{a \neq y\}$, then $\pi(p_t) = 1\{p_t \geq 1/2\}$. Other examples are similarly easy to calculate.

Now, if we knew $p_t$ and played actions $a_t = \pi(p_t)$, then  $\E_{y_t \sim p_t} \ell(\pi(p_t), y_t) \leq \E_{y_t \sim p_t} \ell(a_\star , y_t)$ at every time step $t$. Hence, this strategy would yield a related version of \Cref{eq:constant_action_regret} in an idealized world where we knew the data generating process.

As we described above, \algname lets us act as if the $y_t$ were such ideal random samples. Thus, we can generate predictions $p_t$ where we can effectively assume that $y_t$ was sampled from $p_t$ and choose actions $a_t = \pi(p_t)$ that yield low regret. To see why this suffices, assume that we generate predictions satisfying the following indistinguishability guarantees,
\begin{align}
\label{eq:dec_oi}
\left| \sum_{t=1}^T \ell(\pi(p_t), y_t) - \sum_{t=1}^T \E_{\yt_t \sim p_t}[\ell(\pi(p_t), \yt_t)] \right| &\leq \cR_1(T) \\ 
\sup_{a}\left| \sum_{t=1}^T \ell(a, y_t) -  \sum_{t=1}^T \E_{\yt_t \sim p_t}[\ell(a, \yt_t)] \right| &\leq  \cR_2(T),
\label{eq:loss_oi}
\end{align}
where $\cR_1(T)$ and $\cR_2(T)$ are both $o(T)$. Then, by the first inequality in \Cref{eq:dec_oi}, 
\begin{align*}
     \sum_{t=1}^T \ell(\pi(p_t), y_t) \leq \sum_{t=1}^T \E_{\yt_t \sim p_t}[\ell(\pi(p_t), \yt_t)] + \cR_1(T)\,.
\end{align*}
Furthermore, by definition of $\pi$, we also know that for any $a$ and time step $t$,
\begin{align*}
    \E_{\yt \sim p_t}[\ell(\pi(p_t), \yt_t)] \leq \E_{\yt \sim p_t}[\ell(a, \yt_t)].
\end{align*}
Lastly, the second indistinguishability guarantee in \Cref{eq:loss_oi}, yields, 
\begin{align*}
\sum_{t=1}^T \E_{\yt \sim p_t}[\ell(a, \yt_t)] \leq \sum_{t=1}^T \ell(a, y_t) + \cR_2(T).
\end{align*}
Putting these three equations together, we get the desired regret guarantee \Cref{eq:constant_action_regret},
\begin{align*}
    \sum_{t=1}^T \ell(a_t,y_t) \leq \min_{a_\star} \sum_{t=1}^T \ell(a_\star, y_t) + \cR_1(T) +\cR_2(T)\leq  \min_{a_\star} \sum_{t=1}^T \ell(a_\star, y_t) + o(T)\,.
\end{align*}
With this analysis in mind, the only thing that is left to find a way of getting the desired indistinguishability guarantees. These we can achieve using \algname and \Cref{lemma:oi}.

Recall that the goal is to be indistinguishable with respect to the functions $\ell(\pi(p_t),y_t)$ and $\ell(a, y_t)$ from \Cref{eq:loss_oi,eq:dec_oi}. Let $B=\sup_a |\ell(a,1) - \ell(a,0)|$ and define, 
\begin{align*}
    \Phi(x,p) =   \begin{bmatrix} \ell(\pi(p), 1)-\ell(\pi(p), 0) & B\end{bmatrix}^\top
\end{align*}
With this choice of $\Phi$, the discrete derivatives of our functions can we written as $\langle v,  \Phi(x,p)\rangle$ for fixed vectors $v$ (the $v$ do not depend on $x$ or $p$),
\begin{align*}    
    \ell(\pi(p_t),1) - \ell(\pi(p_t),0) &= \left\langle \begin{bmatrix}
        1 \\ 0
    \end{bmatrix}, \Phi(x,p)\right\rangle, \\
    \ell(a,1) - \ell(a,0) &= \left\langle \begin{bmatrix}
        0\\ B^{-1}(\ell(a,1) - \ell(a,0))
    \end{bmatrix}, \Phi(x,p)\right\rangle.
\end{align*}
Since $\sup_{x,p}\norm{\Phi(x,p)}^2 \leq 2B^2$ and the vectors $v$ in the equations above have norm at most 1, \Cref{lemma:oi} implies that \algname produces predictions satisfying, 
\begin{align*}
\left| \sum_{t=1}^T \ell(\pi(p_t), y_t) - \sum_{t=1}^T \E_{\yt \sim p_t}\ell(\pi(p_t), y_t) \right| &\leq \sqrt{2TB^2}\\ 
\sup_{a}\left| \sum_{t=1}^T \ell(a, y_t) -  \sum_{t=1}^T \E_{\yt \sim p_t}\ell(a, y_t) \right| &\leq  \sqrt{2TB^2}\,.
\end{align*}
Therefore, we get that 
$$
\sum_{t=1}^T \ell(a_t, y_t) - \min_{a^*}\sum_{t=1}^T \ell(a_\star, y_t) \leq  2B\sqrt{2T}
$$
Note that this last result is a purely deterministic statement. It holds with probability 1 over the realized sequence of $y_t$. Furthermore, we made no assumptions (e.g. convexity) on the loss $\ell$ other than the fact that it is bounded and that $\ell(\pi(p_t),1) - \ell(\pi(p_t),0)$ is continuous in $p$. Furthermore, for simplicity, we considered the case where there are no context vectors $x_t$. However, the same ideas generalize to that setting as we will see in a moment.

The presentation in this section follows the analysis first developed for the offline setting in \citet{omnipredictors,gopalan2022loss} and extended to the online context in \citet{garg2024oracle}, \citet{okoroafor2025near}, \citet{noarov2025high}, and \citet{kernelOI}. We note that much earlier work by \cite{foster2006calibration} derived a similar relationship between having a small norm for \Cref{eq:moment-vector} and low regret in the square-loss.

% The beauty of the \algname and OI techniques is that, while the realized outcomes are completely arbitrary, by looking into the past, they can ``cook the books'' and let us  pretend that $y$ was effectively sampled from the distribution $\Ber(p)$ on average. By ensuring this notion of indistinguishability, we were able to argue that we were choosing near optimal actions. This unexpected connection gives us a deterministic algorithm for online risk minimization that assumes nothing about revealed sequence of outcomes $y_t$.

\section{Linear Classes and Online Learning}\label{sec:linear_bandits}

Making predictions on par with a constant action is one thing, but what if you want to outperform more sophisticated prediction functions? For example, we might want to choose actions that perform as well as those computed as functions of a provided context vector $x_t$. In equations, we'd like to make predictions such that
$$
	\sum_{t=1}^T \ell(a_t, y_t) \leq \min_{h \in \cH} \sum_{t=1}^T \ell(h(x_t), y_t) + o(T)\,.
$$
Here, $a_t$ is the action chosen at time $t$ and $\cH$ is a class of functions mapping features to actions. 
Richer classes $\cH$ lead to stronger guarantees. For instance, if $\cH$ is the class of all linear functions, $h(x) = \langle w, x\rangle + a$, the best function in $\cH$ is at least as good as the best fixed action that we considered in the previous section. It can perhaps be considerably better if the optimal action is easily predictable from the provided context $x_t$. 

A simple modification of the previous \algname algorithm enables us to achieve low regret in this more challenging setting. We summarize this result in the following lemma.
\begin{lemma}
\label{lemma:general_loss_min}
Assume that, 
\begin{align}
\label{eq:dec_oi_h}
\left| \sum_{t=1}^T \ell(\pi(p_t), y_t) - \sum_{t=1}^T \E_{\yt \sim p_t}[\ell(\pi(p_t), \yt_t)] \right| &\leq \cR_1(T) \\ 
\sup_{h \in \cH}\left| \sum_{t=1}^T \ell(h(x_t), y_t) -  \sum_{t=1}^T \E_{\yt \sim p_t}[\ell(h(x_t), \yt_t)] \right| &\leq  \cR_2(T)\,.
\label{eq:loss_oi_h}
\end{align}
Then, 
\begin{align*}
    \sum_{t=1}^T \ell(\pi(p_t), y_t) \leq \min_{h \in \cH} \sum_{t=1}^T \ell(h(x_t), y_t) + \cR_1(T) + \cR_2(T)
\end{align*}
\end{lemma}

The lemma above generalizes the argument we saw in the last section to work for any loss function and class $\cH$. In particular, note that the conditions of the lemma are direct generalizations of  \Cref{eq:dec_oi,eq:loss_oi} that we saw led to low excess risk with respect to the best, fixed action in hindsight $a_\star$. In particular, \Cref{eq:dec_oi_h} is identical to \Cref{eq:dec_oi}, and \Cref{eq:loss_oi_h} is the direct analogue of \Cref{eq:loss_oi} where we've swapped out $\ell(a,y_t)$ for $\ell(h(x_t),y_t)$.
Here, $\pi(p_t)$ is defined the same way as before in \Cref{eq:pi_def}. For any $h$, it satisfies:
$$
\E_{\yt \sim p_t}[\ell(\pi(p_t), \yt_t)] =  \min_a \E_{\yt \sim p_t}[\ell(a, \yt_t)] \leq \E_{\yt \sim p_t}[\ell(h(p_t), \yt_t)])\,.
$$
The proof is also identical to the argument we saw before. For any function $h \in \cH$,
\begin{align*}
\sum_{t=1}^T \ell(\pi(p_t), y_t) &\leq \sum_{t=1}^T \E_{\yt \sim p_t}[\ell(\pi(p_t), \yt_t)] + \cR_1(T)  &\text{ (By \Cref{eq:dec_oi_h})}\\
& \leq \sum_{t=1}^T \E_{\yt \sim p_t}[\ell(h(p_t), \yt_t)] + \cR_1(T)  &\text{(By definition of $\pi(p_t)$)}  \\ 
& \leq \sum_{t=1}^T \ell(h(p_t), \yt_t) + \cR_1(T)+\cR_2(T)\,. &\text{(By \Cref{eq:loss_oi_h})}
\end{align*}
Since these inequalities hold for any $h$, they also hold for the best function in $\cH$, proving the lemma.

Now, let's see how we can operationalize these ideas via \algname. To simplify notation and the algorithm, let's focus our attention to the special case where $\ell$ is the squared loss, $\ell(p,y) = (y-p)^2$ and our action is simply predicting a $p \in [0,1]$. Let us assume that features $x$ are vectors in $\R^d$ with norm at most $B$, $\norm{x} \leq B$, and that we are comparing against the class of linear predictors with norm at most $M$, $\cH = \{\langle x, w \rangle: \norm{w} \leq M\} $. 

With these choices, the regret minimization problem is equivalent to minimizing the Brier score of predictions with respect to the best low-norm linear prediction computable from the full sequence,
\begin{align*}
  \lim_{T\rightarrow 0} \left| \frac{1}{T} \sum_{t=1}^T (p_t - y_t)^2 - \min_{w: \|w\| \leq M} \frac{1}{T}\sum_{t=1}^T(y_t- \langle x_t, w\rangle)^2 \right| =0\, . 
\end{align*}
Note that in this case $\ell(p,1) - \ell(p,0) = 1-2p$ for any $p$, and 
$$
\pi(p_t) = \argmin_{p} \E_{\yt \sim p_t}(\yt_t - p)^2 = p_t\,.
$$
Suppose, we run \algname (\Cref{alg:dmm}) with the feature mapping
\begin{align}
\label{eq:linear_regression_feature_map}
    \Phi(x,p) = \begin{bmatrix} 1 & p & x  \end{bmatrix}^\top\,. 
\end{align}
The discrete derivative of these functions can again be written as $\langle v, \Phi(x,p) \rangle$ for fixed vectors $v$,
\begin{align*}    
    \ell(\pi(p_t),1) - \ell(\pi(p_t),0) &= 1-2p_t  =\left\langle \begin{bmatrix}
        1 \\ -2 \\0
    \end{bmatrix}, \Phi(x,p)\right\rangle, \\
    \ell(\langle w, x_t \rangle,1) - \ell(\langle w, x_t \rangle,0) &= 1- 2\langle w, x_t \rangle  =\left\langle \begin{bmatrix}
        1 \\ 0 \\-2w
    \end{bmatrix}, \Phi(x,p)\right\rangle.
\end{align*}
Therefore, by \Cref{lemma:oi},  since $\norm{\Phi(x,p)}^2$ is uniformly bounded by  $2+ B^2$, and the vectors $v$ have (squared) norms bounded by $5$ and $1 + 4\norm{w}^2 \leq 1 + 4M^2$, \algname generates predictions $p_t$ that, 
\begin{align*}
\left| \sum_{t=1}^T \ell(\pi(p_t), y_t) - \sum_{t=1}^T \E_{\yt \sim p_t}[\ell(\pi(p_t), \yt_t)] \right| & \leq \sqrt{5T(2+ B^2)} \\ 
\sup_{w: \norm{w} \leq M}\left| \sum_{t=1}^T \ell(\langle w, x_t\rangle, y_t) -  \sum_{t=1}^T \E_{\yt \sim p_t}[\ell(\langle w, x_t\rangle, \yt_t)] \right| &\leq \sqrt{(1 + 4M^2)T(2+B^2)}\,.
\end{align*}
Applying \Cref{lemma:general_loss_min}, we get that the excess loss is bounded by the sum of these two upper bounds: 
\begin{align*}
    \sum_{t=1}^T (y_t - p_t)^2 \leq \min_{w: \norm{w} \leq M} \sum_{t=1}^T (y_t- \langle w, x_t \rangle)^2 + 2\sqrt{T (5 + 4M^2)(2+B^2)}\,.
\end{align*}

We note that there are other algorithms that achieve similar performance for linear prediction. Notably, given a step size parameter $\alpha > 0$, the online gradient method sets, 
\begin{align*}
	w_t &= w_{t-1} - \alpha (p_{t-1}-y_{t-1}) x_{t-1},	
\end{align*}
and predicts $p_t = \langle w_t, x_t \rangle$. This sequence of predictions achieves,
\begin{align*}
    \sum_{t=1}^T (y_t - p_t)^2 \leq \min_{w: \norm{w} \leq M} \sum_{t=1}^T (y_t- \langle w, x_t \rangle)^2 + \frac{M^2}{2 \alpha} + \frac{1}{2} \alpha M^2B^4 T\,.
\end{align*}
Setting $\alpha$ appropriately, the online gradient method has a similar $\sqrt{T}$ excess risk bound.
See, for example, Theorem 1 in \citet{zinkevich2003online}. As we can see by this expression, the incremental gradient method is also not making predictions about the future. Though we don't know how to derive online gradient descent as a form of \algname, it isn't too far away in spirit or in functional form.

% Indeed, we can compute the predictions computed by our \algname procedure to those produced by gradient descent. With the feature map from \Cref{eq:linear_regression_feature_map}, \algname performs anticorrelation search on the function
% \begin{align*}
%     S_t(p) = \langle x_t, \sum_{i=1}^{t-1} x_i(y_i-p_i) \rangle + \sum_{i=1}^{t-1} (y_i -p_i) + p \sum_{i=1}^{t-1} p_i (y_i-p_i).
% \end{align*}
% We can compute roots of this in closed form, yielding the defensive forecasting predictions 
% \begin{align*}
%     p_t = \frac{1}{\sum_{i=1}^{t-1}p_i(y_i-p_i)} \left( \langle x_t, \sum_{i=1}^{t-1} x_i(p_i-y_i) \rangle + \sum_{i=1}^{t-1} (p_i - y_i) \right).
% \end{align*}
% For the gradient method with constant stepsize, initialized with $w_0=0$, one can show by induction that the predictions are
% \begin{align*}
%     p_t^{(\mathrm{gm})} = \langle x_t, \alpha \sum_{i=1}^t x_i ( \langle 
% \end{align*}

% Comparing this expression to that of the online gradient method, we see that the defensive forecasting forecasts are a linear function of the online gradient descent predictions, 
% $
% p_t^{(\mathrm{df})} = \alpha_t p_t^{(\mathrm{gm})} + \beta_t.
% $

That said, \algname has an interesting extensible property that is not as obvious for all online learning methods: We can concatenate two \algname guarantees together just by concatenating the associated maps $\Phi$. In the next set of examples, we describe algorithms for generating calibrated predictions $p_t$. This will imply \algname algorithms that can yield predictions that both have low-regret \emph{and} are calibrated.

Before leaving regret minimization, it's worth recalling the original question we raised in the introduction. Are these prediction results good? The important point here is that in all problems with sublinear regret, the produced predictions are only as good as the baseline they are compared to. In this case, the baseline is a constant linear prediction function that has access to all of the data in advance. \emph{If} a linear function provides good predictions, \emph{then} \algname makes comparably good predictions. Once we make a commitment of how predictions will be evaluated and what they will be compared against, we can run \algname. But we reiterate there is no way to guarantee in advance whether the baseline itself provides a good fit to the data.

\section{\algname in Kernel Spaces}\label{sec:kdmm}

Before we introduce algorithms for calibration, we first show how to perform \algname with infinite dimensional $\Phi$. Note that the function $S_t$ used in \Cref{alg:dmm} is only a function of dot products between $\Phi$ at various $x$ and $p$. Hence, if we only had access to a \emph{kernel function} $k$ that computed such dot products $k( (x,p), (x',p')) = \langle \Phi(x,p), \Phi(x',p')\rangle$, we could still run \algname. We simply replace all dot products in the subroutine \Cref{alg:dmm} with kernel evaluations. This enables us to work with very rich, high-dimensional function spaces in a computationally-efficient manner. 

\begin{algorithm}[t]
\caption{\algname for Matching Empirical Moments in RKHS}\label{alg:kdmm}
\begin{algorithmic}[1]
\State Define $S_t(p) =\sum_{i=1}^{t-1} k((x_t,p),(x_i,p_i)) (y_i -p_i)$.
\State Run anticorrelation search (\Cref{alg:anti-search}) on $S_t$ to find $p_t$.
\end{algorithmic}
\end{algorithm}
For completeness, we write this out as \Cref{alg:kdmm}. The
analysis of \Cref{alg:kdmm} is almost exactly the same as that of \Cref{alg:dmm}. The main difference is the feature map is now potentially infinite-dimensional. It maps a point $(x,p)$ into a function from $\cX \times [0,1]$ into $\R$. Specifically,  $\Phi_k(x,p)$ is the \emph{function} with $\Phi_k(x,p)(x',p'):=k(x,p,x',p')$. With this mapping, we have a natural dot product between the functions $\Phi_k$, $\langle \Phi_k(x,p), \Phi_k(x',p')\rangle=k(x,p,x',p')$. 

This notation generalizes the finite-dimensional presentation thus far. Any feature map $\Phi(x,p)$ has a corresponding kernel given by its inner products $\langle \Phi(x,p), \Phi(x',p')\rangle$. The function space is the space of linear combinations of the coordinates. Any such function can be written as $f(x,p) = \langle \vartheta, \Phi(x,p)\rangle$ for some vector $\vartheta$. In this regard, \Cref{alg:dmm} is a special case of the kernelized version \Cref{alg:kdmm}.

The kernelized version of \algname was the one originally presented by \citet{vovk2007k29}. In particular, he proves the following proposition, which is a generalization of our presentation in Section~\ref{sec:dmm}. For completeness, we provide the proof in the Appendix.

\begin{proposition}
\label{prop:k29_feature_map}
Suppose $\cH$ is a reproducing kernel Hilbert space with kernel $k : \cX \times [0,1] \times \cX \times [0,1] \rightarrow \R$ that is continuous with respect to its second argument. Then, for all $h\in \cH$, \Cref{alg:kdmm} guarantees 
\begin{align*}
\left|\sum_{t=1}^T h(x_t,p_t) (y_t - p_t)\right| \leq \|h\|_\cH \sqrt{\sum_{t=1}^T(y_t-p_t)^2 k(x_t,p_t,x_t,p_t)}.
\end{align*}
This implies that if $\sup_{(x,p)}k(x,p,x,p)$ is uniformly bounded by $M^2$, then for any function $f(x,p,y)$ such that, $f(x,p,1) - f(x,p,0) = \langle v, \Phi_k(x,p) \rangle_{\cH}$, we have that
\begin{align*}
   \left|\frac{1}{T} \sum_{i=t}^T f(x_t,p_t,y_t) -  \frac{1}{T} \sum_{i=t}^T \E_{\yt_t\sim \Ber(p_t)}[f(x_t,p_t,\yt_t)]  \right| \leq \frac{M \|v\|_{\cH}}{\sqrt{T}}\,.
\end{align*}
\end{proposition}

\section{Calibration}\label{sec:calibration}

% As discussed in the introduction, \algname enables simple and strong algorithms for various different problems. Here, we illustrate how if can be used to guarantee calibration. 

% Calibrated predictions have debatable utility, but many communities of forecasters demand calibrated predictions for various reasons. 

A sequence of probabilistic predictions is perfectly \emph{calibrated} if the fraction of all times where the prediction $p_t$ is equal to $\alpha$ and $y_t$ is equal to $1$ is approximately $\alpha$. Sometimes one needs to state this definition multiple times in multiple ways for it to sink in. It rains on 30 percent of the days where a calibrated weather forecaster predicts a 30 percent chance of rain. Calibrated forecasts are those where the predictions correspond to observed frequencies. The utility of calibration comes in terms of communication: when a calibrated forecast declares a percent chance, this is reflected by a relative correspondence between the frequency of similar outcomes. When a calibrated forecast asserts a probability $p$, their track record shows that such events happen $p$-fraction of the time. 

Let us put aside the utility of calibrated predictions for a moment and turn to \algname procedures that generate calibrated predictions. Note that a set of predictions is calibrated if for all $\alpha \in [0,1]$
$$
    \frac{\sum_{t=1}^T \ind{y_t=1, p_t=\alpha}}{ \sum_{t=1}^T \ind{p_t=\alpha} } = \alpha 
$$
Rearranging this expression, an equivalent form of calibration is satisfying that for all $\alpha$
\begin{align}
\label{eq:perfect_calibration_og}
  \sum_{t=1}^T (y_t-\alpha) \ind{p_t=\alpha} = 0\,.
\end{align}
Since the indicator function is equating $p_t$ with $\alpha$ we can write this condition equivalently as
\begin{align}
\label{eq:perfect_calibration}
  \sum_{t=1}^T (y_t-p_t) \ind{p_t=\alpha} = 0\,.
\end{align}
This formulation looks like Defensive Moment Matching, where the feature function is the infinite dimensional function indexed by $\alpha \in [0,1]$, $\Phi(p) =(\ind{p_t=\alpha})$. The tricky part is that the indicator functions are not continuous in the predictions $p$. People have introduced a number of different definitions of approximate calibration that measure ``closeness" to perfect calibration. We will now show how several of the most popular ones in the literature can be efficiently attained by \algname.

%Similar to our risk minimization results, we apply the exact same algorithm, \Cref{alg:dmm}. The key idea is instantiate the algorithm with a particular choice of feature mapping $\Phi$ and use a key property of \algname that we summarize below.
%\begin{lemma}\label{lemma:calibration}
%Let $F(x,p,y) = (y-p)\Phi(x,p)$ and suppose that for some constant $C$,
%\begin{align*}
%   \norm{ \sum_{i=t}^T (y_t-p_t) \Phi(x_t,p_t)} \leq C\sqrt{T}\,.
%\end{align*}
%Then, for any function $f(x,p) =\langle v, \Phi(x,p) \rangle$, we have
%\begin{align}
%\label{eq:df_calibration}
%   \left|\sum_{i=t}^T f(x_t,p_t) (y_t - p_t)  \right| \leq  C \norm{v} \sqrt{T}\,.
%\end{align}
%\end{lemma}
%The lemma states that if the norm of $\sum_{t=1}^T F(x_t,p_t,y_t)$ is $\cO(\sqrt{T})$ for $F$ equal to $(y-p)\Phi(x_t,p_t)$ then the sequence of predictions $p_t$ is approximately calibrated with respect to all functions $f$ that lie in the span of the features $\Phi$. To see why they are approximately calibrated, note that the left hand side of \Cref{eq:df_calibration} identifies with that of \Cref{eq:perfect_calibration} if we let $f(x,p) = 1\{p=\alpha\}$. The approximately comes from the fact that this quantity is not exactly zero, but rather $o(T)$.
%
%We can recover various different definitions of calibration by varying the functions $f$.

Perhaps the easiest interesting definition that is achievable is that of \emph{smooth 
 (or weak) calibration} by~\citet{kakade2004deterministic}. The smooth calibration error of a sequence of predictions $p$ is,
\begin{align*}
    \SM(p )= \sup_{f \in \cF_{\mathrm{Lip}}} \left|\sum_{t=1}^T f(p_t)(y_t - p_t)\right|\,,
\end{align*}
where $\cFLip$ is the (infinite) set of 1-Lipschitz functions from $[0,1]$ to $[0,1]$. To see the relationship between smooth calibration and calibration, consider a continuous approximation of the indicator function $\ind{p=\alpha} \approx w_\epsilon(p-\alpha)$ where 
$$
	w_\epsilon(x) := \begin{cases}
		1+\frac{x}{\epsilon} & x \in [-\epsilon,0]\\
		1-\frac{x}{\epsilon} & x \in [0,\epsilon]\\
		0 & \text{otherwise}
	\end{cases}\,.
$$
Then the Lipschitz constant of $w_\eps(x)$ is $1/\epsilon$. We'll return to this example in detail momentarily.

It remains to show that we can make smoothly calibrated predictions using \algname. This was first proven by~\citet{vovk2007k29}. Following his notation, we can define a norm on the space of differentiable functions from $[0,1]$ to $[0,1]$ as
$$
	\|f\|_{\mathrm{FS}}^2 = \left( \int_0^1 f(t) dt \right)^2 + \int_0^1 |f'(t)|^2 dt\,.
$$
The completion of this normed space is a Hilbert space called the Fermi-Sobolev space. In particular, every 1-Lipschitz function is differentiable almost everywhere and has a finite Fermi-Sobolev norm.  It turns out that the Fermi-Sobolev space is an RKHS, and based on the work of \citet{wahba1975smoothing}, Vovk derives a simple form for its kernel function
\begin{align}
\label{eq:fs_kernel}
   k_\mathrm{FS}(p,p') = \tfrac{1}{2}\min(p,p')^2 + \tfrac{1}{2}\min(1-p,1-p')^2 + \tfrac{5}{6}\,.
\end{align}

We can thus use ideas from the previous section to efficiently produce predictions that are smoothly calibrated. In particular, there exists a feature map such that any function $f \in \cFLip$ has FS norm at most $\sqrt{2}$. Running the version of \algname from \Cref{alg:kdmm} with this kernel, we will get a sequence of predictions with $\SM(p) \leq \sqrt{2T}$.

Moreover, for any $\alpha \in [0,1]$ and $\epsilon \leq \min(\alpha,1-\alpha)$, define $h_{\epsilon,\alpha}(p)=w_\epsilon(p-\alpha)$. The functions $h_{\epsilon,\alpha}$ map $[0,1]$ to $[0,1]$ and have small Fermi-Sobolev norm because $\int_{-1}^1 w_\epsilon(t) dt=\epsilon$ and $\int_{-1}^1 |w_\epsilon'(t)|^2 dt = 2/\epsilon$,  so we have 
\begin{equation}\label{eq:w_eps_norm}
\|h_{\epsilon,\alpha}\|_{\mathrm{FS}}= \sqrt{\epsilon^2 + \frac{2}{\epsilon}}\,.
\end{equation}

Let's now use this result to derive a randomized algorithm with calibrated predictions. The key will be to generate a sequence of predictions using Defensive Forecasting and then reveal predictions by rounding them to an equispaced grid. The following proposition is a simplification of the procedure presented in \citet{kakade2004deterministic}.

\begin{proposition}\label{prop:rounded_calibration}
Generate a sequence of predictions by running \Cref{alg:kdmm} with the Fermi-Sobolev kernel defined in \Cref{eq:fs_kernel}. Reveal the forecasts  
\begin{align*}
\mathrm{Round}_{N}(p_t) = 
\frac{1}{N} \begin{cases}
\lfloor N p_t \rfloor \text{ with probability } \operatorname{frac}(N p_t)\\ 
\lfloor N p_t \rfloor+1 \text{ with probability }1-\operatorname{frac}(N p_t) \\ 
\end{cases}
\end{align*}
where $\lfloor x \rfloor$ denotes the floor of $x$ and $\operatorname{frac}$ denotes the fractional part. Then with probability $1-\delta$ in the rounding procedure, we have for all integers $n$ between $0$ and $N$
\begin{align*}
 \left|\sum_{t=1}^T \ind{\mathrm{Round}_{N}(p_t) = \tfrac{n}{N} }\left(\tfrac{n}{N}-y_t\right) \right| & \leq \sqrt{T} \left(\sqrt{\frac{8N+2}{3}} + \sqrt{2\log(2 (N+1) / \delta)}  \right) + \frac{T}{2N}\,.
\end{align*}
\end{proposition}

Note that the rounding algorithm here maps the prediction to its closest point on a grid with spacing $1/N$. For instance, if $N=20$, the rounding operation maps $p_t=.89$ to $v'=.9$ with probability $.2$ and to $v=.85$ with probability $.8$. If $N$ is chosen to be equal to $T^{1/3}$, the corresponding bound has a regret of $T^{2/3}$ that is the regret accrued by most sequential calibration algorithms~\citep{Dagan25}.

\begin{proof}
This operation is, in expectation, a smooth function of the prediction. In particular, for any integer $n$ between $0$ and $N$,
\begin{align*}
     \E[\ind{\mathrm{Round}_{N}(p_t) = \tfrac{n}{N}}(p_t-y_t)] =  w_{\tfrac{1}{N}}(p_t-\tfrac{n}{N}) (p_t-y_t) \,.
\end{align*}
Consequently, for any fixed sequence of predictions and outcomes, 
\begin{align*}
\sum_{t=1}^T \E [1\{\mathrm{Round}_{N}(p_t) = \tfrac{n}{N}\}(p_t-y_t)]  =  \sum_{t=1}^T w_{\tfrac{1}{N}}\left(p_t-\tfrac{n}{N}\right) (p_t -y_t)
\end{align*}
Now define 
$$
Y_t = \ind{\mathrm{Round}_{N}(p_t) = \tfrac{n}{N}}(p_t-y_t) -  w_{\tfrac{1}{N}}(p_t-\tfrac{n}{N}) (p_t -y_t)\,.
$$ Since each prediction is rounded independently and $\E[Y_t] = 0$, the partial sums of the random variables $Y_t$ form a martingale. The Azuma-Hoeffding inequality thus implies with probability $1-\delta$: 
\begin{align*}
 \left|\sum_{t=1}^T 1\{\mathrm{Round}_{N}(p_t) = \tfrac{n}{N}\}(p_t-y_t) \right|& \leq \left| \sum_{t=1}^T w_\epsilon(p_t-\tfrac{n}{N}) (p_t -y_t)\right| + \sqrt{2T\log(2 / \delta)}\, .
\end{align*}
By, \Cref{prop:k29_feature_map} we have 
$$
\left| \sum_{t=1}^T w_\epsilon(p_t-\tfrac{n}{N}) (p_t -y_t)\right| \leq  \sqrt{T} \cdot \left\{ \sup_{p} \sqrt{k_{\mathrm{FS}} (p,p)} \right\} \cdot \left\|h_{\frac{1}{N},\frac{n}{N}}\right\|_\mathrm{FS}\,. 
$$
\citet{vovk2007k29} shows $\sup_{p} k_{\mathrm{FS}} (p,p) \leq \frac{4}{3}$.  \Cref{eq:w_eps_norm} yields $\left\|h_{\frac{1}{N},\frac{n}{N}}\right\|_\mathrm{FS} = \sqrt{N^{-2} + 2N}$. Tying these bounds together with inequality above and taking a union bound over all $1\leq n\leq N$ gives.
$$
\left |\sum_{t=1}^T 1\{\mathrm{Round}_{\Delta}(p_t) = \tfrac{n}{N}\}(p_t-y_t) \right| \leq \sqrt{\tfrac{4}{3} T} \sqrt{\tfrac{1}{2} + 2N} + \sqrt{2T\log(2(N+1) / \delta)}\,.
$$
The theorem follows because $|p_t - \frac{n}{N}| \leq \tfrac{1}{2N}$ when $\mathrm{Round}_{\Delta}(p_t) = \tfrac{n}{N}$.
\end{proof}

There are many popular definitions of calibration, and few can agree on what the right one is. \citet{qiao2024distance} gives a laundry list of different notions. They in particular show that smooth calibration is within $O(\sqrt{T})$ of several notions of \emph{distance to calibration}, studied by \citet{blasiok2023unifying}. Hence, any method achieving smooth calibration also yields decent distance to calibration.

An additional key feature of \algname with kernels is that you can easily satisfy multiple objectives at once. If you want to produce predictions that are smootly calibrated but also have a Brier Score comparable to a linear prediction function, you can use the kernel
$$
 k(x,p,x',p') =1 + k_{\mathrm{FS}} (p,p') + pp' +  \langle x, x' \rangle 
$$
Running \Cref{alg:kdmm} with this kernel would yield predictions satisfying the following inequalities, 
\begin{align*}
    \sum_{t=1}^T (y_t -p_t)^2 - \min_{w: \norm{w} \leq M} (p_t-\langle x_t,w\rangle)^2 &\leq 5\sqrt{(1+M) (1+ \max_{1\leq t\leq T} \norm{x_t})T} \\ 
    \sup_{f \in \cF_{\mathrm{lip}}}| \sum_{t=1}^t f(p_t)(y_t-p_t) |& \leq 2\sqrt{T}. 
\end{align*}
Perhaps instead, you'd like a predictor that compares well to smooth functions of the revealed contexts $x_t$. Then you'd use
$$
 k(x,p,x',p') =1 + k_{\mathrm{FS}} (p,p') + pp' +  \exp\left(-\gamma \|x- x'\|^2\right)\,.
$$

Perhaps the main question before the analyst is why they want calibrated predictions at all. \citet{foster2021forecast} argue that philosophically, it is better for forecasters to return calibrated probababilistic predictions since consumers of those forecasts can interpret the forecasts in terms of probabilities. Calibration allows analysts to assert, as a FiveThirtyEight headline put it\footnote{https://fivethirtyeight.com/features/when-we-say-70-percent-it-really-means-70-percent/}, ``When We Say 70 Percent, It Really Means 70 Percent.'' 

But we emphasize that beyond this property of shared interpretation, calibration doesn't mean much. As many have emphasized before, a set of predictions can be perfectly calibrated and essentially useless for the purposes of decision-making. If a sequence of outcomes is a string of random bits with an equal number of $1$'s and $0$'s, then predicting $p_t=1/2$ will be calibrated and it will be the best constant prediction as well. The value of any prediction algorithm is only good if its assumptions about the future turn out to be correct (say, that a constant prediction has low prediction error). Unfortunately, no algorithm can guarantee the future will look like the past.

\section{$\ell_\infty$ \algname \& Expert Prediction}\label{sec:experts}

We conclude the paper with two examples showing the versatility of the \algname paradigm. In this section, we demonstrate how to derive \algname algorithms that compete with \emph{expert predictions}. The setup is similar to what we've seen thus far. At each time $t$, we observe a context vector $x_t$. We also receive the predictions of $N$ experts, $f_j(x_t)$. Our goal is to make a prediction so that
\begin{equation}\label{eq:basic-experts}
	\lim_{T \rightarrow \infty} \frac{1}{T}\sum_{t=1}^T \ell(p_t, y_t) -  \inf_{1\leq j\leq N} 
 \frac{1}{T} \sum_{t=1}^T \ell(f_j(x_t), y_t)  \leq 0 \,.
\end{equation}
where $\ell$ is a prespecified loss function that measures prediction error.

Though he doesn't provide an explicit algorithm, the existence of a \algname approach to this problem was proven by~\citet{vovk2007defensive}. Here we provide a simple algorithm based on anticorrelation search. The core idea is to minimize the $\ell_\infty$ norm of the sum of vectors,
$$\left\| \sum_{t=1}^T F(x_t,p_t,y_t) \right\|_{\infty}.$$
To do this, we minimize the soft-max function, a smooth upper bound to the max function,
$$
 \left\| \sum_{t=1}^T F(x_t,p_t,y_t) \right\|_{\infty}  \leq \log \sum_{j=1}^N \exp\left( \sum_{t=1}^T F_j(x_t,p_t,y_t)\right)\,.
$$
where $F_j(x_t,p_t,y_t)$ is the $j$th coordinate of $F(x_t,p_t,y_t) \in \R^N$. Note that this soft-max surrogate admits the following  recursion,
$$
M_t = \log \sum_{j=1}^N \exp\left( \sum_{s=1}^t F_j(x_s,p_s,y_s) \right) = M_{t-1} +  \log  \sum_{i=1}^N \alpha_i \exp\left(F_i(x_t,p_t,y_t) \right)
$$
where,
$$
	\alpha_i = \frac{\exp( \sum_{s=1}^{t-1}  F_i(x_s,p_s,y_s) ) }{\sum_{j=1}^N \exp( \sum_{s=1}^{t-1} F_j(x_s,p_s,y_s) )} \,.
$$
If we can thus find a $p_t$ such that 
\begin{align*}
    \sup_{y \in \{0,1\}}  \log  \sum_{i=1}^N \alpha_i \exp\left(F_i(x_t,p_t,y) \right) \leq 0,
\end{align*}
we will prove that $M_t \leq M_0$, which gives us a bound on the maximum we desire. 

This is a \algname strategy. The goal again is to make predictions so that a particular function is negative no matter the actualized outcomes. \Cref{alg:linf_df} produces a sequence of predictions that follows this strategy. It works for a generalization of the experts problem, finding a sequence of predictions so that
$$
	\frac{1}{T}\max_{1\leq j \leq N} \sum_{t=1}^{T} F_j(x_t,y_t,p_t)  \rightarrow 0
$$
as $T$ goes to $\infty$. To specialize it to the experts problem, we simply let the entries in $F$ correspond to the gaps to each expert,
\begin{align*}
    F(x_t,p_t,y_t) = \begin{bmatrix}
        \ell(p_t, y_t) - \ell(f_1(x_t), y_t) \\ 
        \dots \\
        \ell(p_t, y_t) - \ell(f_N(x_t), y_t) 
    \end{bmatrix}\,.
\end{align*}

% That is, rather than controlling the size of the $\ell_2$ norm of a vector-valued function $F$, we now aim to control its $\ell_\infty$ norm. This problem is a generalization of \Cref{eq:basic-experts} where $F_j(x_t,y_t,p_t) = \ell(p_t, y_t)-\ell(f_j(x_t), y_t)$. By making predictions to minimize the sum of the exponentials of $F_j$, we get predictions that make the maximum of $F_j$ small.

\begin{algorithm}[t]
\caption{$\ell_\infty$ \algname}\label{alg:linf_df}
\begin{algorithmic}[1]
\State Define  $$Q_{jt} =  \sum_{s=1}^{t-1} F_j(x_s,p_s,y_s)\,.$$
\State Define $$\alpha_{jt} = \frac{\exp(Q_{jt} ) }{ \sum_{j=1}^N \exp(Q_{jt})}\,.$$
\State Define $$S_t(p) = \sum_{j=1}^N \alpha_{jt} \exp\left( F_j(x_{t},p,1)\right) - \sum_{j=1}^N \alpha_{jt} \exp\left( F_j(x_{t},p,0) \right)$$
\State Run anticorrelation search (\Cref{alg:anti-search}) on $S_t$ to find $p_t$.
\end{algorithmic}
\end{algorithm}

While \Cref{alg:linf_df} does not work for arbitrary $F_j$, it works for all $F_j$ satisfying the following Assumption. We will describe cases (e.g. log loss) where this Assumption holds in the sequel.

\begin{assumption}\label{assumption:bounded}
For all $j$ the functions $F_j$ satisfy, 
$$p\exp\left( F_j(x,p,1)  \right) + (1-p) \exp\left( F_j(x,p,0)  \right) \leq 1,$$ 
for all $(x,p)$.
\end{assumption}

\begin{proposition}\label{prop:linf_df} Under~\Cref{assumption:bounded}, \Cref{alg:linf_df} returns a sequence satisfying
$$
	\max_{1\leq j \leq N} \sum_{s=1}^{T} F_j(x_s,p_s,y_s) \leq \log(N)\,.
$$
\end{proposition}

\begin{proof}
Let
$$
	V_{t} = \sum_{j=1}^N \exp\left(  \sum_{s=1}^{t} F_j(x_s,p_s,y_s) \right)\,.
$$
Since $\log V_t \geq  \max_j  \sum_{s=1}^{t} F_j(x_s,p_s,y_s)$, it suffices to show that $V_T \leq N$.

We proceed, as usual, by induction. Note that $V_0=N$. We now show $V_T \leq V_{T-1}$. First, for simplicity of notation, let $\alpha_j = \alpha_{j,T-1}$ in what follows. Then we have
\begin{align*}
V_T=&y_T \sum_{j=1}^N \exp\left(\sum_{s=1}^{T-1} F_j(x_s,p_s,y_s) +F_j(x_T,1,p_T) \right)\\
&\qquad\qquad	+(1-y_T) \sum_{j=1}^N \exp\left(\sum_{s=1}^{T-1} F_j(x_s,p_s,y_s) + F_j(x_T,0,p_T)\right)\\
	=&V_{T-1} \left\{y_T\sum_{j=1}^N \alpha_{j} \exp\left( F_j(x_T,1,p_T) \right) + (1-y_T) \sum_{j=1}^N \alpha_{j} \exp\left( F_j(x_T,0,p_T) \right) \right\}\,.
\end{align*}
It suffices to show that the term inside the curly brackets is less than or equal to 1. Indeed, we have:
\begin{align*}
& y_T\sum_{j=1}^N \alpha_{j} \exp\left(F_j(x_T,1,p_T)  \right) + (1-y_T) \sum_{j=1}^N \alpha_{j} \exp\left( F_j(x_T,0,p_T)  \right)\\
	\leq & \sup_{q\in [0,1]} q\sum_{j=1}^N \alpha_j \exp\left( F_j(x_T,1,p_T)  \right) + (1-q) \sum_{j=1}^N \alpha_j \exp\left( F_j(x_T,0,p_T)  \right) \\
	= & p_T \sum_{j=1}^N \alpha_j \exp\left( F_j(x_T,1,p_T)  \right) + (1-p_T) \sum_{j=1}^N \alpha_j \exp\left( F_j(x_T,0,p_T)  \right)\;.
	\end{align*}
The equality in line three holds because this is how we chose $p_T$.  Now, for any $p\in [0,1]$, we have
\begin{align*}
	& p\sum_{j=1}^N \alpha_j \exp\left( F_j(x_T,1,p)  \right) + (1-p) \sum_{j=1}^N \alpha_j \exp\left( F_j(x_T,0,p)  \right) \\
	=& \sum_{j=1}^N \alpha_j  \left\{p\exp\left( F_j(x_T,1,p)  \right) + (1-p) \exp\left( F_j(x_T,0,p)  \right) \right\}\leq 1\,.
\end{align*}
The final inequality here follows from \Cref{assumption:bounded}.
\end{proof}

Note that the bound \Cref{prop:linf_df} is independent of $T$. Let's now apply this result to the problem of prediction with expert advice. As was the case in Sections~\ref{sec:risk_minimization} and~\ref{sec:linear_bandits}, \algname makes good predictions whenever there is a single expert that makes good predictions.

\subsection{Squared Loss}
When the loss function $\ell$ is the squared loss, we set,
$$
	F_j(x,p,y) = \lambda\{ (p-y)^2 - (f_j(x)-y)^2 \},
$$
for a constant $\lambda$. Our goal is to show that this family of $F_j$ satisfy \Cref{assumption:bounded} for all $\lambda \in[0,2]$. 

Here we'll make use of a special case of Hoeffding's Lemma:
$$
	p \exp(a) + (1-p) \exp(b) \leq \exp\left( pa + (1-p) b + \frac{(b-a)^2}{8} \right)\,.
$$
Then we have
\begin{align*}
	p F_j(x,p,1) + (1-p) F_j(x,p,0) &= -\lambda (p-f_j(x))^2\\
	(F_j(x,p,1)-F_j(x,p,0))^2 &= 4 \lambda^2 (p-f_j)^2\,.
\end{align*}
This gives
\begin{align*}
p\exp\left( F_j(x,p,1)  \right) + (1-p) \exp\left( F_j(x,p,0)  \right) &\leq \exp\left( \left\{-\lambda + \tfrac{1}{2} \lambda^2 \right\} (p-f_j)^2 \right)\,.
\end{align*}
which is less than or equal to $1$ if $\lambda \leq 2$. Hence, \Cref{assumption:bounded} holds and by \Cref{prop:linf_df}, \Cref{alg:linf_df} returns predictions $p_t$ satisfying
$$
	\frac{1}{T}\sum_{t=1}^T (p_t-y_t)^2 -  \inf_{1\leq j\leq N} \frac{1}{T}\sum_{t=1}^T (f_j(x_t)-y_t)^2 \leq \frac{\log(N)}{2T}\,.
$$

\subsection{Log Loss}
For log loss regret minimization, 
$$
	F_j(x,p,y) = - y (\log p - \log f_j(x)) - (1-y) (\log(1-p) -\log(1-f_j(x)) \,.
$$
Plugging in the definition, we immediately see
\begin{align*}
	p\exp\left( F_j(x,p,1)  \right) + (1-p) \exp\left( F_j(x,p,0)  \right) = f_j(x)+ (1-f_j(x))=1\,.
\end{align*}
Hence, \Cref{assumption:bounded} again holds and by \Cref{prop:linf_df}, \Cref{alg:linf_df} returns predictions $p_t$ satisfying
$$
	\frac{1}{T}\sum_{t=1}^T \ell(p_t,y_t) -  \inf_{1\leq j\leq N} \frac{1}{T}\sum_{t=1}^T \ell(f_j(x_t),y_t) \leq \frac{\log(N)}{T}\,.
$$
As a final remark, all of these results also hold in the setting where experts see the forecaster's predictions and make potential refinements. That is, all of the results in this section hold when $f_j$ take as input both $x$ \emph{and} $p$. Vovk calls these ``second-guessing'' experts. \citet{lee2022online} analyze a similar exponential weights approach to the one we present above that guarantees $\cO(\sqrt{T\log(N)})$ regret with respect to general losses. 

\section{Quantiles}\label{sec:quantiles}

\algname strategies can also be applied to real-valued predictions. In this section, we focus on the particular example of quantile prediction. This will allow us to draw connections to online conformal prediction and also introduce a new algorithmic technique, adapted from~\citet{foster1999proof}, for \algname problems with discontinuities.

Suppose that we want to make predictions about the quantiles of a sequence of real-valued outcomes $y_t \in (\Ymin,\Ymax]$. In an online setting, to start, we'd like to find predictions $p_t$ such that on average, we are predicting an accurate estimate of the quantile:
\begin{equation}\label{eq:dumb-quantile}
	\lim_{T\rightarrow \infty} \left|\frac{1}{T} \sum_{t=1}^T 1\{y_t \leq p_t\}  - q \right| = 0\,.
\end{equation}
This condition at first glance seems intimately related to the goal of estimating a quantile since it can only be achieved if $y_t$ is at most $p_t$ a $q$ fraction of the time.

The standard \algname strategy achieves this goal. Define $S_t:= \sum_{i=1}^{t-1} 1\{y_i \leq p_i\} - q (t-1)$. If $S_t \leq  0$,  predict $p_t = \Ymax$, otherwise predict $p_t = \Ymin$. At every time $t$, the forecasts satisfy
\begin{align*}
    \sup_{y \in (\Ymin, \Ymax]} (1\{y \leq p_t\} - q) \sum_{i=1}^{t-1} (1\{y_i \leq p_i\} - q) \leq 0\, .
\end{align*}
This is true because if $S_t =  \sum_{i=1}^{t-1} (1\{y_i \leq p_i\} - q) \leq 0$ predicting $p_t = \Ymax$ guarantees that $1\{y \leq p_t\} - q \geq 0$ and hence the product is negative. If $S_t \geq 0$, then choosing $p_t = \Ymin$ ensures that $1\{y \leq p_t\} - q \leq 0$ and we get the same invariant. By induction, we thus again have: 
\begin{align*}
    \left(\sum_{t=1}^T 1\{y_i \leq p_i\} - q \right)^2 \leq  \sum_{t=1}^T (1\{y_i \leq p_i\} - q )^2 \leq T.
\end{align*}

This algorithm is perhaps the most perplexing instance of \algname yet. It achieves \Cref{eq:dumb-quantile} without ever needing to look at the outcomes $y_t$. We can interpret the algorithm as predicting infinity when it outputs $\Ymax$ and negative infinity when it outputs $\Ymin$. By predicting plus or minus infinity, the sequence $S_{t+1} = 1\{y_t \leq p_t\} + S_{t}$ is deterministic. That means the algorithm guarantees what forecasters call ``marginal coverage'' \emph{without looking at any data}. This has little to do with what we'd like from a quantile estimator, and it definitely has nothing to do with uncertainty quantification. \citet{bastani2022practical} point out a similar issue in marginal guarantees for online prediction intervals. 

Instead, for this particular evaluation metric, \algname is a deterministic algorithm that computes an approximation of the number $q$ by averaging $0$s and $1$s. Let's look at what happens in each round. Define,
$$
	x_t = \frac{1}{t} \sum_{i=1}^{t} 1\{y_i \leq p_i\} \,.
$$
If $x_t\leq q$, $p_{t+1}$ is set to $\Ymax$ and $x_{t+1} $ is set to $(1-1/t) x_t + (1/t)$. If $x_t<q$,  $p_{t+1}$ is set to $\Ymin$ and $x_{t+1}$ is set to $(1-1/t) x_t$. This rewriting of the algorithm let's us do a slightly more refined analysis,  showing it in fact achieves a $1/T$ rate.

\begin{proposition}
Let $q \in [0,1]$. Set $x_1=0$ and let $x_{t+1} = (1-1/t) x_t + (1/t) \mathbf{1}\{ x_t \leq q \}$ for $t>1$. Then,
$$
\left|\frac{1}{T}\sum_{t=1}^T (1\{p_t \leq y_t\} - q) \right|=| x_T - q | \leq \frac{\max\{q,1-q\}}{T-1}\,.
$$
\end{proposition}

\begin{proof}
Let $e_t = x_t-q$. We proceed by induction. When $t=1$, $e_1 \leq q$. Now assume 
$$
	| e_t | \leq \frac{\max\{q,1-q\}}{t-1}\,.
$$
Then there are two cases. If $x_t \leq q$, we have by the inductive hypothesis $-\max\{q,1-q\} \leq e_t\leq 0$ and hence
$$
	t e_{t+1} = (t-1) e_t + (1-q) \in [-\max\{q,1-q\},1-q]\,.
$$
Similarly, if $x_t>q$, we have $0 \leq e_t\leq \max\{q,1-q\}$ and hence
$$
	t e_{t+1} = (t-1) e_t -q \in [-q, \max\{q,1-q\}]\,.
$$
completing the proof.
\end{proof}

The analysis of this simple deterministic algorithm in this section so far shows that some online quantile metrics are too easy to game.  Let us now develop quantile algorithms with potentially more meaningful conditional guarantees.

\subsection{Randomized Forecasts}

To do this, we take a small detour and present a generalization of the meta-algorithm for \algname that we saw in \Cref{sec:meta_algorithm}.
Let $F(x,p,y)$ be a vector-valued function that is possibly discontinuous in $p$ and let $\mathcal{K}$ be a set of distributions over the outcome space $\cY$. At every time step, the defensive forecaster will sample their prediction from a distribution $\Delta^p_t$. We'll also assume that the outcomes $y_t$ are drawn from a distribution $\Delta^y_t \in \mathcal{K}_\cY$. Suppose that we know
\begin{equation}\label{eq:foundational_randomized}
\sup_{\Delta^y_t \in \mathcal{K}}  \underset{p_t \sim \Delta^{p}_t, y_t \sim \Delta^y_t}{\E} \left[ \left\langle F(x_t,p_t,y_t),  \sum_{s=1}^{t-1} F(x_s,p_s,y_s)\right\rangle \right] \leq \frac{\eps_t}{2}.
\end{equation}
Then, applying the same induction argument,
\begin{align*}
\norm{ \underset{p_t \sim \Delta^p_t,  y_t \sim \Delta^y_t}{\E}  \sum_{t=1}^T  F(x_t,p_t,y_t)}^2 \leq  \sum_{t=1}^T  \norm{ \underset{p_t \sim \Delta^p_t,  y_t \sim \Delta^y_t}{\E} F(x_t,p_t,y_t)}^2  +\eps_t.
\end{align*}
If the $\eps_t$ are summable ($\sum_{t=1}^\infty \eps_t \leq C$) and $\norm{F(x,p,y)}$ is bounded by $M$, we would get that:
\begin{equation}\label{eq:diagonal-bound-randomized}
\norm{\frac{1}{T}\underset{p_t \sim \Delta^p_t,  y_t \sim \Delta^y_t}{\E}  \sum_{t=1}^T  F(x_t,p_t,y_t)} \leq \frac{M}{\sqrt{T}} + \frac{C}{T}.
\end{equation}
\Cref{eq:foundational_randomized} is direct analogue of the fundamental \algname condition (\Cref{eq:foundational}) with the difference that now it holds in expectation rather than deterministically. The key advantage compared to the analysis in \Cref{sec:meta_algorithm} is that $F$ no longer needs to be continuous $p$. However, ensuring that the functions $F$ have norms that grow sublinearly as per \Cref{eq:diagonal-bound-randomized} enables us to apply \algname for quantile prediction.

\subsection{Conditional Online Quantile Estimation}
Take the case where $F(x,p,y) = \Phi(x_t, p_t)(1\{p_t \leq y_t\} -q)$. If we can guarantee that our predictions satisfy \Cref{eq:foundational_randomized} and hence \Cref{eq:diagonal-bound-randomized},  then this means that
\begin{align}
\label{eq:conditional_quantiles}
    \left|\frac{1}{T} \sum_{t=1}^T \underset{p_t \sim \Delta^p_t,  y_t \sim \Delta^y_t}{\E} [f(x_t,p_t)(1\{p_t \leq y_t\} -q)] \right| \leq \norm{v} \left( \frac{M}{\sqrt{T}} + \frac{C}{T} \right)\,. 
\end{align}
for all functions $f$ that can be written as $f(x,p) = \langle v, \Phi(x,p) \rangle$ (this again follows by linearity and Cauchy-Schwarz).
%\begin{align*}
%     \left|\sum_{t=1}^T \underset{p_t \sim \Delta^p_t,  y_t \sim \Delta^y_t}{\E} [f(x_t,p_t)(1\{p_t \leq y_t\} -q)] \right| &= \left|\sum_{t=1}^T \underset{p_t \sim \Delta^p_t,  y_t \sim \Delta^y_t}{\E} [\langle v, \Phi(x,p) \rangle(1\{p_t \leq y_t\} -q)] \right| \\ 
%     & = \left| \left\langle v, \sum_{t=1}^T \underset{p_t \sim \Delta^p_t,  y_t \sim \Delta^y_t}{\E}\Phi(x,p) (1\{p_t \leq y_t\} -q) \right\rangle \right| \\ 
%     & \leq \norm{v} \norm{\frac{1}{T}\underset{p_t \sim \Delta^p_t,  y_t \sim \Delta^y_t}{\E}  \sum_{t=1}^T  F(x_t,p_t,y_t)} \,.
%\end{align*}
\Cref{eq:conditional_quantiles} is a potentially much more meaningful guarantee than the marginal one considered in \Cref{eq:dumb-quantile}. If for instance, we let $\mathcal{E}_i$ be a collection of subsets of $\cX \times [\Ymin, \Ymax]$ and define 
$$\Phi(x,p) = (p_t, 1\{(x,p) \in \mathcal{E}_1 \}, \dots, 1\{(x,p) \in \mathcal{E}_N\})^\top  \in \R^{N+1} $$
we get that 
$$\sup_{1 \leq i \leq N} |\sum_{t=1}^T \E[ 1\{(x_t,p_t) \in \mathcal{E}_i\}(1\{p_t \leq y_t\} -q)]| \leq o(T).$$
This guarantee cannot be achieved by always predicting $\Ymax$ or $\Ymin$ as before.

It remains to show that we can efficiently achieve the fundamental \algname condition from \Cref{eq:foundational_randomized}. The algorithm we present now is a simplified version of that in \cite{kernelOI} albeit with worse constants.

Assume that at every round, the features $x_t$ in $\cX$ are chosen arbitrarily. Having seen $x_t$, the forecaster selects a distribution $\Delta^p_t$ over forecasts $p \in \cY  = (\Ymin, \Ymax]$ and then Nature, knowing $\Delta_t^p$, selects a distribution $\Delta_y^t$ over the same interval $\cY$ from the class $\mathcal{K}_L$ of $L$-Lipschitz distributions. We say that a distribution is $L$-Lipschitz if its CDF satisfies,
\begin{align*}
    |\Pr_{y \sim \Delta^y_t}[y \leq v] - \Pr_{y \sim \Delta^y_t}[y \leq v']| \leq L \cdot |v-v'|.
\end{align*}
Using a trick developed in \cite{foster1999proof} and extended in \citet{foster2021forecast}, we show that one can always find a distribution $\Delta_t^p$ supported on two close together points $p_{t,1}$ and $p_{t,2}$ such that:
\begin{equation}\label{eq:foundational_randomized_q}
\sup_{\Delta^y_t \in \mathcal{K}}  \underset{p_t \sim \Delta^{p}_t, y_t \sim \Delta^y_t}{\E} \left[ \left\langle \Phi(x_t,p_t)(1\{y_t \leq p_t\} -q),  \sum_{s=1}^{t-1} \Phi(x_s,p_s)(1\{y_s \leq p_s\} -q)\right\rangle \right]  \leq \eps_t
\end{equation}
for any $\eps_t > 0$. We can find this distribution using a randomized variartion of anticorrelation search. Let, $$S_t(p) = \langle \Phi(x_t,p_t), \sum_{s=1}^{t-1} \Phi(x_s,p_s)(1\{p_s \leq y_s\} -q) \rangle\,.$$ 
If $S_t(\Ymin) \geq 0$, then predicting $p_t = \Ymin$ guarantes that $1\{y_t \leq p_t\} - q \leq 0$, thereby satisfying \Cref{eq:foundational_randomized_q}. Otherwise if $S_t(\Ymax) \leq 0$, the predicting $p_t = \Ymax$ ensures that  $1\{y_t \leq p_t\} - q \geq 0$ which also implies the inequality. In both these cases, $\Delta_t^p$ is just a point mass. 

If neither of these cases are true, it must then be the case that $S_t(\Ymin) < 0 < S_t(\Ymax)$ and that the function $S_t(p)$ jumps from negative to positive at some point between $\Ymin$ and $\Ymax$. That is there must be at least two points $p_{t,1}, p_{t,2} \in [\Ymin, \Ymax]$, that are $\gamma_t$ close for any $\gamma_t > 0$, $|p_{t,1} - p_{t,2}|\leq \gamma_t$ and have opposite signs, $S_t(p_{t,1})<0< S_t(p_{t,2})$. The full procedure for generating forecasts is given in \Cref{alg:quant}.

Now, let $\Delta_t^p$ be the distribution over $\cY$ that outputs $p_{t,1}$ with probability $\tau$ and $p_{t,2}$ with probability $1-\tau$ where $\tau \in [0,1]$ solves, 
\begin{align}
\label{eq:hedging}
    \tau S_{t}(p_{t,1}) + (1-\tau) S_{t}(p_{t,2}) = 0\,.
\end{align}
Such a $\tau$ exists because $S_{t}(p_{t,1})$ and $S_{t}(p_{t,2})$ have opposite signs. By definition of $S_t$, the expression inside the supremum on the left hand side of \Cref{eq:foundational_randomized_q} is equal to
\begin{align*}
\underset{p_t \sim \Delta^{p}_t, y_t \sim \Delta^y_t}{\E} \left[ (1\{y_t \leq p_t\} -q) S_t(p_t)\right] ],.
\end{align*}
And, with our choice of $\Delta_t^p$, we can rewrite this as:
\begin{align*}
 \tau \cdot  S_t(p_{t,1}) \underset{y_t \sim \Delta^y_t}{\E} \left[ (1\{y_t \leq p_{t,1}\} -q) \right] +  (1- \tau) \cdot  S_t(p_{t,1}) \underset{y_t \sim \Delta^y_t}{\E} \left[ (1\{y_t \leq p_{t,2}\} -q) \right] ],.
\end{align*}
If we now add and subtract $\tau S_{t}(p_{t,1})(1\{y_t \leq p_{t,2}\} -q)$, this becomes
\begin{align*}
\tau \cdot  S_t(p_{t,1}) \underset{y_t \sim \Delta^y_t}{\E} \left[ 1\{y_t \leq p_{t,1}\} - 1\{y_t \leq p_{t,2}\}\right] +  [\tau S_{t}(p_{t,1}) + (1-\tau) S_{t}(p_{t,2}) ]\underset{y_t \sim \Delta^y_t}{\E} \left[ (1\{y_t \leq p_{t,2}\} -q) \right] \,.
\end{align*}
The term on the right side is zero by \Cref{eq:hedging}. And the term on the left can be made small for all choices of $\Delta^y_t$ by setting $\gamma_t = |p_{t,1}- p_{t,2}|$ to be small, 
\begin{align*}
\tau \cdot  S_t(p_{t,1}) \underset{y_t \sim \Delta^y_t}{\E} \left[ 1\{y_t \leq p_{t,1}\} - 1\{y_t \leq p_{t,2}\}\right]  & =  \tau \cdot  S_t(p_{t,1})     |\Pr_{y_t \sim \Delta^y_t}[y \leq p_{t,1}] - \Pr_{y \sim \Delta^y_t}[y_t \leq p_{t,2}]|  \\ 
& \leq |S_{t}(p_{t,1})| \cdot L \cdot  |p_{t,1}- p_{t,2}| = |S_{t}(p_{t,1})| \cdot L \cdot \gamma_t\,.
\end{align*}
In particular, letting $\gamma_t = 1 / (10 t^2 |S_{t}(p_{t,1})|)$, we get \Cref{eq:foundational_randomized_q} with $\eps_t = L / (10t^2)$. Note that $\sum_{t=1}^\infty \eps_t \leq L$. Tying this together with our meta-analysis, we get the following formal result:
\begin{theorem}
Suppose $\cH$ is a reproducing kernel Hilbert space with kernel $k$ and assume outcomes $y_t$ are drawn from a $L$-Lipschitz distribution $\Delta_t^y$. Then, for all $f\in \cH$, \Cref{eq:quantile_alg} guarantees 
\begin{align*}
\left|\sum_{t=1}^T \underset{p_t \sim \Delta^{p}_t, y_t \sim \Delta^y_t}{\E} \left[ f(x_t,p_t) (1\{y_t\leq p_t\} - p_t) \right]\right| \leq \|f\|_\cH \sqrt{L + \sum_{t=1}^T  \underset{p_t \sim \Delta^{p}_t, y_t \sim \Delta^y_t}{\E} k((x_t,p_t),(x_t,p_t))}\,.
\end{align*}
In particular, this implies that if $k(x,p,x',p') = \Phi(x,p)^\top \Phi(x',p')$ where $\Phi(x,p)$ is an explicitly computable feature map with $\sup_{(x,p)} \norm{\Phi(x,p)}^2 \leq M$, then, for any $f(x,p) = \langle v, \Phi(x,p)\rangle$:
\begin{align*}
    \left|\sum_{t=1}^T \underset{p_t \sim \Delta^{p}_t, y_t \sim \Delta^y_t}{\E} \left[ f(x_t,p_t) (1\{y_t\leq p_t\} - p_t) \right]\right| \leq \norm{v} \sqrt{L + MT}.
\end{align*}
\end{theorem}
    
\begin{algorithm}[t]
\caption{\algname for Online Conditional Quantile Estimation}\label{alg:quant}
\begin{algorithmic}[1]
\State Define $S_t(p) = \sum_{i=1}^{t-1} k((x_t,p),(x_i,p_i)) (1\{y_i \leq p_i\} -q)$
\If{$S_t(\Ymin) \geq  0$}
\State Predict $p_t = \Ymin$. 
\ElsIf{$S_t(\Ymax) \leq 0$}
\State Predict $p_t=\Ymax$. 
\Else
\State Run binary search on $S_t(\cdot)$ to find $p_{t,1}$ and $p_{t,2}$ such that 
\begin{align*}
	S_t(p_{t,1}) < 0 < S_t(p_{t,2}) \text{ with } |p_{t,1} - p_{t,2}| \leq \frac{1}{10 t^2 |S_t(p_{t,1})|}
\end{align*}
\State Set $\tau = \frac{|S_t(p_{t,2})|}{|S_t(p_{t,1})|+|S_t(p_{t,2})|}$.
\State Predict $p_t=p_{t,1}$ with probability $\tau$ and $p_{t,2}$ with probability $1 - \tau$.
\EndIf
\end{algorithmic}
\label{eq:quantile_alg}
\end{algorithm}

As before, the algorithm only depends on evaluating inner products and hence we can generalize it to work for any kernel function $k(x,p, x',p')$ as per our discussion in \Cref{sec:kdmm}. Our presentation thus far where $k(x,p, x',p') = \langle \Phi(x,p), \Phi(x',p') \rangle$ for an explicit $\Phi$ is just a special case. We also note that the algorithm is completely hyperparameter free. We don't need to know the lipschitz constant $L$ ahead of time, even though it does play a role in the analysis.

The algorithm is also computationally efficient. If $\Phi(x,p)\in \R^{d}$ is finite dimensional, by maintaining the counter $\sum_{s=1}^t \Phi(x_s,p_s)(1\{y_s \leq p_s\} -q)$ we can get the run time to be $\widetilde{\cO}(d)$ at time $t$. If we instead compute inner products implicitly via the kernel function, the run time becomes $\widetilde{\cO}(t \cdot \mathrm{time}(k))$ where $\mathrm{time}(k)$ is an upper bound on the time it takes to evaluate the kernel. 

Furthermore, we emphasize that this is still a fully adaptive or adversarial setting where the distribution over outcomes can depend on the algorithm's choice of $\Delta_t$ as well as the features $x_t$ and the entire history of observations. At two distinct times $t$ and $s$, the distributions $\Delta^y_t$ and $\Delta^y_s$ can be completely different. \citet{gupta2021OnlineML} use a very similar randomization trick to the one we use above to derive online mean, moment, and quantile calibration algorithms with $\sqrt{T}$ regret.
\citet{bastani2022practical} also apply this randomization trick to arrive at an exponential-weights style algorithm for online prediction intervals with group conditional gaurantees. 

\newcommand{\cA}{\mathcal{A}}
\section{Batch Learning with \algname}\label{sec:online_to_batch}

Given that they are designed to perform well in worst-case settings, one might think \algname algorithms are overly conservative. However, \algname also makes high quality predictions when data is randomly sampled. If you have an online algorithm which accrues low regret in some metric for arbitrary sequences, the same algorithm also achieves low excess risk in the situation when data sequence is sampled i.i.d. from a fixed probability distribution. 

To make this precise, we say an prediction method is \emph{online algorithm} if it computes predictions $p_t$ sequentially from a stream of data $\{(x_t,y_t)\}_{t=1}^T$ where $x_t \in \cX$ and $y_t \in \cY$. We say an prediction method is a \emph{batch algorithm} if given a dataset $S = \{(x_i,y_i)\}_{i=1}^n$ of $n$ examples, it produces a potentially randomized algorithm $\cA_S$ such that maps an arbitrary data point $x$ to a prediction $p$.

There is a generic procedure, called \emph{online to batch conversion} which turns an online algorithm into a batch algorithm with parallel theoretical guarantees. The \algname algorithms we've developed output at every round a prediction $p_t$ as a function of the history  $\pi_{<t} = \{(x_1, y_1),\dots, (x_{t-1}, y_{t-1})\}$ and the current $x_t$. Hence, there is some function so that $p_t=\cA_t(x_t,\pi_t)$. 
Denote $f_t$ by the function that maps $x$ to $\cA_t(x,\pi_t)$. These functions $f_t$ are generally not defined explicitly. For instance, \Cref{alg:kdmm} at time $t$ returns $f_t$ which given $x$ outputs the $p$ returned from the anticorrelation search subroutine on the function $\sum_{j=1}^{t-1} k((x,p),(x_j,y_j))(y_j-p_j)$. 

These $f_t$ form the basis of our batch algorithm. Given a data set, run online \algname on the sequence $(x_1,y_1), \dots, (x_n, y_n)$ to produce the sequence of functions $f_1,\dots, f_n$. Define $\cA_S$ to be the algorithm that given $x$, picks $f_i$ uniformly at random from $\{f_1,\dots, f_n\}$ and then predicts $p=f_i(x)$.
This online-to-batch conversion comes with a universal guarantee. 

\begin{proposition}\label{prop:online-to-batch}
Let $\cA$ be any online algorithm that when run on a sequence of data guarantees
\begin{align}
\label{eq:deterministic_regret}
    \sup_{\omega \in \Omega} \left| \sum_{t=1}^T \omega(x_t,p_t,y_t) \right| \leq \cR(T)
\end{align}
for all functions $\omega: \cX \times \cP \times \cY \rightarrow \R$ in some finite set $\Omega$.  Then if $S = \{(x_i,y_i)\}_{i=1}^n$ consists of $n$ examples drawn i.i.d from a fixed distribution $\cD$, the \emph{randomized} algorithm $\cA_S$ satisfies
\begin{align}
\label{eq:expectation_online_to_batch}
  \sup_{\omega \in \Omega}  \left|  \underset{(x,y) \sim \cD, \cA_S \sim \cD^n, p\sim \cA_S(x)}{\E} \omega(x,p,y) \right| \leq \frac{\cR(n)}{n}.
\end{align}
The expectation here is over randomness from the $n$ samples $S \sim \cD^n$, the next draw $(x,y)$, and the internal randomness of $\cA_S$. Moreover, with probability $1-\delta$ over the draw of the dataset $S\sim \cD^n$, we have
\begin{align}
\label{eq:high_probability_online_to_batch}
  \sup_{\omega \in \Omega}  \left|  \underset{(x,y) \sim \cD, p\sim \cA_S(x)}{\E} \omega(x,p,y) \right| \leq \frac{\cR(n)}{n} + 2B\sqrt{\frac{\log(2/\delta) + \log(|\cH|)}{n}}\,,
\end{align}
where $B = \sup_{w,x,p,y} |\omega(x,p,y)|$. 
\end{proposition}

\Cref{prop:online-to-batch} is a straightforward consequence of the Azuma-Hoeffding inequality. We defer the proof to the appendix.

Let's work through a few of applications to see how online algorithms with worst-case performance can achieve near-optimal average case results. Define $\Omega_{\cH}$ to be the set of all functions,
\begin{align*}
 \omega_h(x,p,y) = \ell(p,y) - \ell(h(x_t),y),
\end{align*}
for $h$ in some class $\cH$. Then \Cref{prop:online-to-batch} implies that online risk minimization algorithms that deterministically make predictions $p_t = f_t(x_t)$ satisfying, 
\begin{align*}
    \sup_{h \in \cH}  \left| \sum_{t=1}^T \ell(p_t,y_t)  - \min_{h \in \cH} \sum_{t=1}^T \ell(h(x_t),y_t)  \right|\leq \cR(T),
\end{align*}
can be converted into batch predictors $\cA_S$ such that 
\begin{align}
\label{eq:batch_risk_minimization}
     \underset{(x,y) \sim \cD, \cA_S \sim \cD^n, p\sim \cA_S(x)}{\E}[\ell(p,y)] \leq \min_{h \in \cH} \E_{(x,y)\sim \cD}[\ell(h(x),y)] + \frac{\cR(T)}{T}.
\end{align}
In particular, if we let $\ell$ be the squared loss or log loss this implies that \algname algorithm from \Cref{sec:experts} yields a batch predictor $\cA_S$ satisfying,
\begin{align*}
         \underset{(x,y) \sim \cD, \cA_S \sim \cD^n, p\sim \cA_S(x)}{\E}[\ell(p,y)] \leq \min_{h \in \cH} \E_{(x,y)\sim \cD}[\ell(h(x),y)] + \frac{\log(|\cH|)}{n}.
\end{align*}
This $1/n$ excess error is optimal for both log loss and squared loss. No algorithm can achieve a better bound given $n$ random examples. 

We can also extend online-to-batch conversion to our randomized \algname algorithm for predicting conditional quantiles. In this case, $\cA_t(x_t,\pi_t)$ is a randomized procedure. Hence, we can consider $f_t(x_t)$ to be a distribution $\Delta_t^p$ over $[\Ymin, \Ymax]$ that is supported on 2 points. Still, all the analysis thus far works, we just define our batch prediction for $x$ by first sampling an $f_i$ uniformly and then sampling $p$ from $f_i(x)$. When $f_i$ are distributions over predictions, we write,
\begin{align}
\label{eq:expected_regret}
    \sup_{\omega \in \Omega} \left| \sum_{t=1}^T \E_{y_t \sim \Delta_t^t, p_t \sim f_t(x_t)}\omega(x_t,p_t,y_t) \right| \leq \cR(T)
\end{align}
instead of \Cref{eq:expectation_online_to_batch}.

Let's analyze the performance of online-to-batch conversion in this setting. Let $\cF$ be a finite subset of the set of functions $\{\langle v, \Phi(x,p) \rangle: \norm{v}\leq 1\}$ where $\sup_{(x,p)}\norm{\Phi(x,p)}^2 \leq C$. If the conditional distribution over outcomes $y$ given $x$, is $L$-Lipschitz, \Cref{alg:quant} yields predictions $p_t$ such that 
\begin{align*}
    \sup_{f \in \cF}\left| \underset{p_t, y_t}{\E}f(x_t,p_t)(1\{y_t \leq p_t\} -q) \right| \leq \sqrt{L + CT}\,.
\end{align*}
This equation is the same as \Cref{eq:expected_regret} for $\omega_f(x,p,y) = f(x_t,p_t)(1\{y_t \leq p_t\} -q)$. Therefore, applying the online-to-batch result, with probability $1-\delta$ over $S$, the batch version $\cA_S$ satisfies
\begin{align*}
       \sup_{f \in \cF} \left| \underset{(x,y)\sim \cD, p\sim \cA_S}{\E}f(x_t,p_t)(1\{y_t \leq p_t\} -q) \right| \leq \frac{\sqrt{L}}{n} + C\sqrt{\frac{1 + \log(2/\delta)+ \log(|\cF|) }{n}}\,.
\end{align*}
Hence, if the functions $f$ are binary, $$\E[f(x,p)(1\{y_t \leq p_t\} -q)] = (\Pr[y\leq p| f(x,p)] -q) \Pr[f(x,p)=1], $$ and we get that for any $f\in \cF$ where $\Pr[f(x,p)=1] > 0$
\begin{align*}
    \left| \Pr[y\leq p \mid f(x,p)=1] -q\right| \leq \left(\frac{\sqrt{L}}{n} + 2C\sqrt{\frac{1 + \log(2|\cF|/\delta) }{n}}\right) \frac{1}{\Pr[f(x,p)=1]}.
\end{align*}
Statisticians often refer to these results as conditional coverage statements since the probability that $y$ is greater than $p$ is equal to $1-q + o(1)$ conditional on the event that $f(x,p) =1$.

These example applications are by no means exhaustive. One can also apply these results (e.g Lemma 4.1) to derive batch algorithms for other problems like batch outcome indistinguishability or multicalibration with optimal $n^{-1/2}$ rates. One can even apply this toolkit to arrive at new batch algorithms in \emph{performative} contexts where the data is not i.i.d but rather influenced by the choice of forecasts \citep{perdomo2020performative}. See \citet{perdomo2025revisiting}. Surprisingly, simulating an online setting where we sequentially fix prior mistakes suffices for generating predictions from a fixed batch of randomly sampled data. Given how simple they are to design and analyze, this duality makes \algname algorithms useful for batch machine learning problems when forecasters model their data as i.i.d. samples.

\section{Conclusions}\label{sec:conclusions}

Vovk, who has numerous significant results in martingale theory, has essentially shown that any martingale theorem can be turned into a \algname Algorithm. His \algname algorithm K29 \citep{vovk2007k29} adapts the proof of the weak law of large numbers by \citet{Kolmogorov29}. His \algname experts algorithm \citep{vovk2007defensive} uses properties of supermartingales. Work by \citet{Rakhlin2017} similarly derives gradient descent, mirror descent, and generalizations from high-probability tail bounds for the supremum of martingales. Why do martingales generally provide paths to online learning? A martingale has zero correlation between the present and the past. \algname chooses probabilities so that, no matter what the future holds, it will satisfy these martingale conditions. That you can \emph{choose} ``probabilities'' to ensure these conditions is quite remarkable. That it leads to practical algorithms is even more remarkable.

However, as we have repeatedly emphasized, \algname algorithms are only useful when compared to meaningful baselines. This is true for all online algorithms. The metrics themselves define what we think a good prediction should be. If prediction is possible, good prediction is merely a matter of diligent bookkeeping. You don’t have to be Nostradamus. Whether the metrics capture what we need to capture always, therein lies the true element of chance.

\section*{Acknowledgements} The authors would like to thank Siva Balakrishnan, Unai Fischer-Abaigar, Dean Foster, Laurent Lessard, Tengyuan Liang, Christopher Musco, Aaron Roth, and Ryan Tibshirani for several helpful discussions, comments, and suggestions. JCP was generously supported in part by the Harvard Center for Research on Computation and Society and the Alfred P. Sloan Foundation Grant G-2020-13941. BR was generously supported in part by NSF CIF award 2326498, NSF IIS Award 2331881, and ONR Award N00014-24-1-2531.

{\small
\bibliography{ref}
\bibliographystyle{abbrvnat}
}

\appendix

\section{Proof of \Cref{prop:k29_feature_map}}
The proof follows the same arguments we have seen thus far. The main departure is the feature map is now the infinite dimensional function $\Phi_k(x,p)$. For any $f \in \cH$, we have
\begin{align*}
       \left|\sum_{t=1}^T h(x_t,p_t) (y_t - p_t) \right| &= \left|   \sum_{t=1}^T\langle h, \Phi_k(x_t,p_t) \rangle (y_t - p_t) \right| \\
       &= \left|  \left\langle h, \sum_{t=1}^T \Phi_k(x_t,p_t)  (y_t - p_t)  \right\rangle\right| \\
       & \leq \| h \|_\cH \left\| \sum_{t=1}^T \Phi_k(x_t,p_t) (y_t - p_t) \right\|_\cH
\end{align*}
We now bound the second term in this expression as
\begin{align}
\label{eq:inductive_sum}
    \left\| \sum_{t=1}^T \Phi_k(x_t,p_t) (y_t - p_t) \right\|_\cH^2 \leq \sum_{t=1}^T  \| \Phi_k(x_t,p_t) (y_t - p_t) \|_\cH^2 
\end{align}
To verify this expression, we proceed by induction. The case $T=1$ is immediate. For $T>1$, again note that by design \Cref{alg:kdmm} maintains the kernelized invariant
\begin{align}
\label{eq:conditional_hedging}
    \sup_{y \in [0,1]} (y-p_t) \sum_{i=1}^{t-1} k(x_t,p_t, x_i,p_i)(y_i-p_i) =  \sup_{y \in [0,1]} (y-p_t) S_t(p) \leq 0
\end{align}
at every round $t$. Letting $\varphi_i = \Phi_k(x_t,p_t) (y_i - p_i)$, we have
\begin{align*}
    \left\| \sum_{i=1}^{t-1} \varphi_i  + \varphi_t \right\|_\cH^2 &=  \left\| \sum_{i=1}^{t-1} \varphi_i \right\|_\cH^2 + 2 \left\langle \varphi_t, \sum_{i<t} \varphi_i \right\rangle + \left\|  \varphi_t \right\|_\cH^2\\
& =   \left\| \sum_{i=1}^t \varphi_i\right\|_\cH^2 + 2 (y_t - p_t) S_t(p_t)  +\left\|  \varphi_t \right\|_\cH^2\\
& \leq \sum_{i=1}^t \|\varphi_i\|_\cH^2
\end{align*}
Here, we used the definition of $S_t$ and our guarantee from \Cref{eq:conditional_hedging} in the second line. We then applied the inductive hypothesis in the last one. This shows the claim in \Cref{eq:inductive_sum}. The proof then follows by combining these last few inequalities:
\begin{align*}
   \left|\sum_{t=1}^T h(x_t,p_t) (y_t - p_t)\right| &\leq \|h\|_\cH \left\| \sum_{t=1}^T \Phi_k(x_t,p_t)(y_t - p_t) \right \|_\cH \\ 
    & \leq  \|h\|_\cH  \sqrt{\sum_{t=1}^T  \| \Phi_k(x_t,p_t) (y_t - p_t) \|_\cH^2 } \\
    & =  \|h\|_\cH  \sqrt{\sum_{t=1}^T (y_t-p_t)^2 k(x_t,p_t ,x_t,p_t) } 
\end{align*}
The second statement follows from the first by repeating the same argument as in \Cref{lemma:oi}.

\section{Proof of \Cref{prop:online-to-batch}}

Fix any $\omega \in \Omega$ and define $$Z_t= \omega(x_t,f_t(x_t),y_t) - \E_{(x,y) \sim \cD}[\omega(x,f_t(x),y)].$$
Note that since $f_i$ is a deterministic function of $\pi_{<i}$ and since each data point is drawn i.i.d from a fixed distribution, $(x_i, y_i)\sim \cD$, then 
$\E[Z_i | \pi_{<i}] = 0$.
Summing from $i=1$ to $n$, we get that:
\begin{align*}
    \sum_{i=1}^n Z_i =  \sum_{i=1}^n\E_{(x,y) \sim \cD}[\omega(x,f_i(x),y)] - \sum_{i=1}^n \omega(x_i,f_i(x_i),y_i) 
\end{align*}
Rearranging this expression and plugging in the definition of $\cA_S$, 
\begin{align}
    \underset{(x,y) \sim \cD, p\sim \cA_S(x)}{\E}[\omega(x,p,y)] &= \frac{1}{n}\sum_{i=1}^n \underset{(x,y) \sim \cD, p\sim f_i(x)}{\E}[\omega(x,f_i(x),y)] \notag\\
    & = \frac{1}{n} \sum_{i=1}^n \omega(x_i,f_i(x_i),y_i) - \frac{1}{n} \sum_{i=1}^n Z_i 
    \label{eq:martingale_decomp}
\end{align}
By assumption on the online algorithm, 
\begin{align*}
    \left| \frac{1}{n} \sum_{i=1}^n \omega(x_i,f_i(x_i),y_i) \right| \leq \frac{\cR(n)}{n}
\end{align*}
Therefore, since $\E[\sum_{i=1}^n Z_i]=0$, taking an extra expectation over $S=\{(x_i,y_i)\}_{i=1}^n \sim \cD^n$, 
\begin{align*}
      \left| \underset{(x,y) \sim \cD, p\sim \cA_S(x)}{\E}[\omega(x,p,y)] \right| \leq \frac{\cR(n)}{n} \text{ for all } \omega \in \Omega.
\end{align*}
This proves the in expectation guarantee from \Cref{eq:expectation_online_to_batch}. To prove the high probability result, we use the fact that for any fixed $\omega \in \Omega$, $\{Z_t\}_{t=1}^T$ is a martingale difference sequence with $|Z_T|\leq 2B$. Therefore, the Azuma-Hoeffding inequality implies that with probability $1-\delta$,
\begin{align*}
    \left| \sum_{i=1}^n Z_i\right| \leq 2B\sqrt{n\log(2/\delta)}.
\end{align*}
Plugging this high probability bound into \Cref{eq:martingale_decomp} and taking a union bound over $\omega \in \Omega$ proves the high probability statement in \Cref{eq:high_probability_online_to_batch}. The proof for the case where $f_t$ output distributions over $p_t$ and $y_t \sim \Delta_t^t$ follows the exact same argument except we let
\begin{align*}
Z_t= \E_{p_t \sim f_t(x_t),y}[\omega(x_t,p_t,y)|\cX=x] - \E_{(x,y) \sim \cD,p_t\sim f_t(x)}[\omega(x,p_t,y)].
\end{align*}

\section{Kolmogorov's Proof of the Weak Law of Large Numbers}
To save the interested reader the trouble of tracking down Kolmogorov's 1929 paper, written in French and published in the \emph{Proceedings of the Accademia dei Lincei}, we provide his short proof of the weak law of large numbers here.

\begin{theorem}
Let $\xi_i$ be a sequence of random variables and $X_n$ a function of the first $n$ elements of the sequence. Define
$$
\E_k[X_n] := \E[X_n | \xi_1,\ldots,\xi_k]\,.
$$
Then, $\E[(X_n - \E[X_n])^2] \leq \sum_{k=1}^n\E\left[ (\E_k[X_n] - \E_{k-1}[X_n])^2 \right].$
\end{theorem}
In this case, as long as the variance of each of the increments $\E_k[X_n] - \E_{k-1}[X_n]$ is bounded, the variance of $X_n$ is bounded. The proof technique, where diagonal terms of the expectation are cancelled when a square is expanded, motivates the reasoning in \algname.

\begin{proof}
Set 
$$
	Z_{nk}= \E_k[X_n]-\E_{k-1}[X_n]\,.
$$
We have
$$
	X_n - \E[X_n] = \sum_{k=1}^n Z_{nk}\,.
$$
Now,
$$
	\E[Z_{nk} | \xi_1,\ldots,\xi_{k-1}] = 0
$$
and, moreover when $i < k$, 
$$
	\E[Z_{ni} Z_{nk} ] = 0
$$
Hence, the conclusion holds.
\end{proof} 

\end{document}